\documentclass{article}

\usepackage{microtype}
\usepackage{graphicx}
\usepackage{subfigure}
\usepackage{booktabs} %

\usepackage{hyperref}
\hypersetup{
    colorlinks=true,
    linkcolor=Blue,
    citecolor=black!30!OliveGreen,
    filecolor=Maroon,
    urlcolor=Maroon
}

\usepackage[accepted]{icml2023}  %

\usepackage[utf8]{inputenc}
\usepackage{hyperref}
\usepackage{amsmath}
\usepackage{amssymb}
\usepackage{mathtools}
\usepackage{amsthm}
\usepackage{balance}
\usepackage[capitalize,noabbrev,nameinlink]{cleveref}
\usepackage{nicefrac,bm}
\usepackage{algorithm}
\usepackage{multirow}
\usepackage{arydshln}
\usepackage{csquotes}
\usepackage{anyfontsize}
\usepackage{amsmath}
\usepackage{enumitem}
\allowdisplaybreaks

\providecommand{\Comments}{0}  %
\ifnum\Comments=1
\usepackage[colorinlistoftodos,prependcaption,textsize=scriptsize]{todonotes}
\paperwidth=\dimexpr \paperwidth + 4.6cm\relax
\oddsidemargin=\dimexpr\oddsidemargin + 2.6cm\relax
\evensidemargin=\dimexpr\evensidemargin + 2.2cm\relax
\else
\usepackage[disable]{todonotes}
\fi
\newcommand{\mytodo}[1]{\ifnum\Comments=1{#1}\fi}

\newcommand{\tableoftodos}{\ifnum\Comments=1 \listoftodos[Comments/To Do's] \fi}

\renewcommand{\algorithmiccomment}[1]{\bgroup\hfill$\rhd$~#1\egroup}

\definecolor{Gred}{RGB}{219, 50, 54}
\definecolor{Ggreen}{RGB}{60, 186, 84}
\definecolor{Gblue}{RGB}{72, 133, 237}
\definecolor{Gyellow}{RGB}{247, 178, 16}
\definecolor{ToCgreen}{RGB}{0, 128, 0}
\definecolor{myGold}{RGB}{231,141,20}
\definecolor{myBlue}{rgb}{0.19,0.41,.65}
\definecolor{myPurple}{RGB}{175,0,124}

\theoremstyle{plain}
\newtheorem{theorem}{Theorem}[section]
\newtheorem{proposition}[theorem]{Proposition}
\newtheorem{lemma}[theorem]{Lemma}
\newtheorem{corollary}[theorem]{Corollary}

\newtheorem{observation}[theorem]{Observation}
\theoremstyle{definition}
\newtheorem{definition}[theorem]{Definition}

\theoremstyle{remark}

\newcommand{\set}[1]{\left \{ #1 \right \}}

\newcommand{\inparen}[1]{\left ( #1 \right )}
\newcommand{\insquare}[1]{\left [ #1 \right ]}

\newcommand{\cX}{\mathcal{X}}
\newcommand{\cM}{\mathcal{M}}
\newcommand{\cA}{\mathcal{A}}

\newcommand{\ind}{{\bm 1}}
\newcommand{\eps}{\varepsilon}
\newcommand{\N}{\mathbb{N}}
\newcommand{\Z}{\mathbb{Z}}
\newcommand{\R}{\mathbb{R}}

\newcommand{\cD}{\mathcal{D}}

\newcommand{\bv}{{\bm v}}
\newcommand{\bu}{{\bm u}}

\newcommand{\bzero}{\mathbf{0}}

\newcommand{\bx}{{\bm x}}

\newcommand{\bA}{{\bm A}}

\newcommand{\cL}{\mathcal{L}}
\renewcommand{\phi}{\varphi}

\newcommand{\tO}{\widetilde{O}}
\newcommand{\tOmega}{\widetilde{\Omega}}

\newcommand{\cN}{\mathcal{N}}
\newcommand{\cB}{\mathcal{B}}
\newcommand{\cK}{\mathcal{K}}
\newcommand{\oeps}{\overline{\eps}}
\newcommand{\odelta}{\overline{\delta}}

\newcommand{\hbx}{\hat{\bx}}
\newcommand{\htheta}{\widehat{\theta}}

\newcommand{\stable}{\mathrm{stable}}
\newcommand{\poly}{\mathrm{poly}}
\newcommand{\dTV}{d_{\mathrm{tv}}}

\DeclareMathOperator{\argmin}{argmin}

\DeclareMathOperator{\TDLap}{\mathsf{TDLap}}

\DeclareMathOperator{\delsen}{\overline{\Delta}}

\DeclareMathOperator*{\E}{\mathbb{E}}

\renewcommand{\arraystretch}{1.2} %

\newcommand{\DelOutputPert}{\mathsf{DelOutputPert}}
\newcommand{\SCovOutputPert}{\mathsf{SCOutputPert}}
\newcommand{\PhasedERM}{\mathsf{Phased}\text{-}\mathsf{ERM}}
\newcommand{\PhasedSCO}{\mathsf{Phased}\text{-}\mathsf{SCO}}
\newcommand{\StronglyConvexERM}{\mathsf{Strongly}\text{-}\mathsf{Convex}\text{-}\mathsf{ERM}}
\newcommand{\StronglyConvexSCO}{\mathsf{Strongly}\text{-}\mathsf{Convex}\text{-}\mathsf{SCO}}

\icmltitlerunning{On User-Level Private Convex Optimization}

\date{\today}

\begin{document}
	
\twocolumn[
\icmltitle{On User-Level Private Convex Optimization}

\icmlsetsymbol{equal}{*}

\begin{icmlauthorlist}
	\icmlauthor{Badih Ghazi}{equal,google}
	\icmlauthor{Pritish Kamath}{equal,google}
	\icmlauthor{Ravi Kumar}{equal,google}
	\icmlauthor{Raghu Meka}{equal,google,ucla}
	\icmlauthor{Pasin Manurangsi}{equal,googth}
	\icmlauthor{Chiyuan Zhang}{equal,google}
\end{icmlauthorlist}

\icmlaffiliation{google}{Google Research, Mountain View, CA}
\icmlaffiliation{googth}{Google Research, Thailand}
\icmlaffiliation{ucla}{UCLA, Los Angles, CA}

\icmlcorrespondingauthor{Pasin Manurangsi}{pasin@google.com}
\icmlcorrespondingauthor{Pritish Kamath}{pritish@alum.mit.edu}

\icmlkeywords{Machine Learning, Differential Privacy, User-Level}

\vskip 0.3in
]

\printAffiliationsAndNotice{}  %

\begin{abstract}
We introduce a new mechanism for stochastic convex optimization (SCO) with user-level differential privacy guarantees. The convergence rates of this mechanism are similar to those in the prior work of 
\citet{LevySAKKMS21,NarayananME22},
but with two important improvements.  Our mechanism does not require any smoothness assumptions on the loss.  Furthermore, our bounds are also the first where the minimum number of users needed for user-level privacy has no dependence on the dimension and only a logarithmic dependence on the desired excess error. 

The main idea underlying the new mechanism is to show that the optimizers of strongly convex losses have low local deletion sensitivity, along with an output perturbation method for functions with low local deletion sensitivity, which could be of independent interest.
\end{abstract}

\section{Introduction}
Differential Privacy (DP)~\cite{dwork2006calibrating}~is a formal notion that protects the privacy of each user contributing to a dataset when releasing statistics about the dataset. The settings considered in literature have typically involved each user contributing a single ``item'' to the dataset. Thus the most commonly used notion of DP protects the privacy of each item, and we refer to it as \emph{item-level} DP. However, when a dataset contains multiple items contributed by each user, it is essential to simultaneously protect the privacy of all items contributed by any individual user; this notion has come to be known as \emph{user-level} DP.

Convex optimization is one of the most basic and powerful computational tools in statistics and machine learning. In the most abstract setting, each item corresponds to a loss function. The goal is to return a value that achieves as small a loss as possible, either averaged over the data (empirical risk minimization) or the population distribution underlying the data (stochastic convex optimization).

Given its importance, a large body of work has tackled the convex optimization problem under item-level DP (e.g.,~\citet{ChaudhuriM08,ChaudhuriMS11,KiferST12,BassilyST14,BassilyFTT19,WangYX17,FeldmanKT20,AsiFKT21,GopiLL22}) with the optimal risk bounds established in many standard settings, such as when the loss is Lipschitz or strongly convex.
User-level DP has also been studied recently in various learning tasks \cite{LiuSYK020,GhaziKM21}; see also the survey by \citet[][Section  4.3.2]{kairouz2019advances} for its relevance in federated learning, where the question of determining trade-offs between item-level and user-level DP is highlighted. \citet{LevySAKKMS21,NarayananME22} have studied convex optimization with user-level DP; these results have two main drawbacks: they require the loss function to be smooth and they do not achieve good risk bounds in some regime of parameters.  A question in \citet{LevySAKKMS21} was if the smoothness requirement can be removed.  In this paper, we resolve this question in the affirmative by introducing novel mechanisms for convex optimization under user-level DP.  En route, we also improve existing excess risk bounds for a large regime of parameters.

\subsection{Background}

We introduce some notation to state our results concretely.
For $n, m \in \N$, suppose there are $n$ users, and let the input to the $i$th user be $\bx_i := (x_{i, 1}, \dots, x_{i, m})$. %
Two datasets $\bx = (\bx_1, \ldots, \bx_n)$ and $\bx' = (\bx_1', \ldots, \bx_n')$ are said to be \emph{user-level neighbors}, denoted $\bx \sim \bx'$, if there is an index $i_0 \in [n]$ such that $\bx_i = \bx'_{i}$ for all $i \in [n] \smallsetminus \set{i_0}$.%
\footnote{We use \emph{item-level} to refer to the case where $m = 1$.}

We recall the  definition of DP, extended from~\citet{dwork2006our,dwork2006calibrating}; see also \citet{DworkR14,Vadhan17}:

\begin{definition}[(User-Level) Differential Privacy (DP)]
Let $\eps > 0$ and $\delta \in [0, 1]$.
A randomized algorithm $\cM : \cX^{n \times m} \to \mathcal{O}$ is \emph{$(\eps,\delta)$-differentially private} (\emph{$(\eps,\delta)$-DP}) if, for all $\bx \sim \bx'$ and all (measurable) outcomes $E \subseteq \mathcal{O}$, it holds that $\Pr[\cM(\bx)\in E] \le e^\eps \cdot \Pr[\cM(\bx')\in E] + \delta$. 
\end{definition}

Throughout the paper, we assume that $\eps \in (0, 1]$ and $\delta \in (0, 1/2]$, and we will not state this explicitly.

\textbf{Convex Optimization.}
A convex optimization (CO) problem over a parameter space $\cK \subseteq \R^d$ and domain $\cX$, is specified by a \emph{loss function} $\ell : \cK \times \cX \to \R$, that is convex in the first argument. Here, $\ell$ is said to be \emph{$G$-Lipschitz} if all sub-gradients have norm at most\footnote{We use $\|\cdot\|$ to denote the Euclidean, i.e., $\ell_2$-norm.} $G$, i.e., $\|\nabla_\theta \ell(\theta; x)\| \le G$ for all $\theta$, $x$. Moreover, $\ell$ is said to be \emph{$\mu$-strongly convex} if for all $x \in \cX$, $\ell(\theta; x) - \frac{\mu}{2} \|\theta\|^2$ is convex. We consider the case where $\cK \subseteq \R^d$ has $\ell_2$-diameter at most $R$; we use $\cB_d(\theta, r)$ to denote the $\ell_2$ ball of radius $r$ centered at $\theta$.

The \emph{empirical loss} on dataset $\bx = (\bx_1, \ldots, \bx_n)$ is 
\begin{align*}
\cL(\theta; \bx) &\textstyle~:=~ \frac{1}{nm} \sum_{i \in [n]} \sum_{j \in [m]} \ell(\theta; x_{i, j}),
\end{align*}
whereas the \emph{population loss} over a distribution $\cD$ on $\cX$ is 
\begin{align*}
\cL(\theta; \cD) := \E_{x \sim \cD}\left[\ell(\theta; x)\right].
\end{align*}
For a loss function $\ell$ and  dataset $\bx$, let $\theta^*_{\ell, \cK}(\bx)$ denote an element of $\argmin_{\theta \in \cK} \cL(\theta; \bx)$ (ties broken arbitrarily), and let $\theta^*_{\ell, \cK}(\cD)$ be defined similarly. When $\ell, \cK$ are clear from context, we may drop them and simply write $\theta^*(\bx)$ or $\theta^*(\cD)$. When there is no ambiguity in $\bx$ and $\cD$, we may drop them and simply write $\theta^*$. \emph{Empirical risk minimization} (ERM) corresponds to the goal of minimizing $\cL(\theta; \bx)$ and \emph{stochastic convex optimization} (SCO) to the goal of minimizing $\cL(\theta; \cD)$. If $\htheta$ denotes the output of our algorithm, its \emph{excess risk} is defined as $\E[\cL(\htheta; \bx) - \cL(\theta^*; \bx)]$ and $\E[\cL(\htheta; \cD) - \cL(\theta^*; \cD)]$ for ERM and SCO, respectively.

\subsection{Our Results}\label{subsec:our-results}

We provide user-level DP algorithms for both the ERM as well as the SCO problems.  For both problems, we consider the basic case of Lipschitz (including non-smooth) losses and the case of Lipschitz strongly convex losses.

\textbf{DP-ERM.} We give an algorithm for any Lipschitz and convex loss function with excess risk $O\left(\frac{\sqrt{d}}{n\sqrt{m}}\right)$ that works for any $n \geq \tOmega_\eps(1)$. Previously, no user-level DP algorithm was known without a smoothness assumption on the loss function. Even with smoothness, the known algorithm of~\citet{NarayananME22} incurs an additional additive error of $\tO_\eps(1/\sqrt{n})$. In particular, the previous excess risk does \emph{not} converge to zero if we fix the number of users ($n$) and let $m \to \infty$. %
 Concretely, to achieve excess risk $\alpha$,~\citet{NarayananME22} need $n \geq \tOmega_\eps(1/\alpha^2)$. In contrast, we only need a logarithmic dependence of $n \geq \tOmega_\eps(\log(1/\alpha))$. Additionally, for loss functions that are also strongly convex, we improve the excess risk bound to $\tO_\eps\left(\frac{d}{n^2 m}\right)$. Again, no previous user-level DP algorithm was known in this setting (without the smoothness assumption).

\textbf{DP-SCO.} Here, we give algorithms with similar excess risk bounds except with additive terms of $\tO\left(\frac{1}{\sqrt{nm}}\right)$ and $\tO\left(\frac{1}{nm}\right)$ for the convex and strongly convex cases, respectively. These additive terms are known to be tight, even without privacy. Again, previous results~\cite{LevySAKKMS21,NarayananME22} are only known under the smoothness assumption and the excess risk bounds do not converge to zero when $n$ is fixed.

A summary of the previous and new bounds is in \Cref{tab:summary}.

\textbf{Tightness of our Risk Bounds.} Our excess risk bounds are nearly tight for a large regime of parameters. In particular,~\citet{LevySAKKMS21} proved a lower bound of $\Omega\left(\frac{1}{\sqrt{n}} + \frac{\sqrt{d}}{\eps n\sqrt{m}}\right)$ for DP-SCO. It is not hard to extend this to prove a lower bound of $\Omega\left(\frac{1}{nm} + \frac{d}{\eps^2 n^2m}\right)$ for the strongly convex DP-SCO case. These two lower bounds hold for any $n \geq \Theta(\sqrt{d}/\eps)$. For DP-ERM, it is possible to extend these lower bounds to get $\Omega\left(\frac{\sqrt{d}}{\eps n\sqrt{m}}\right)$ and $\Omega\left(\frac{d}{\eps n^2m}\right)$ lower bounds for the convex and strongly convex cases, respectively; however, these DP-ERM lower bounds require an additional assumption that $n = O(d/\eps^2)$. We discuss these lower bounds in more detail in \Cref{app:lb}.

\newcommand{\sst}[1]{\shortstack{#1}}
\renewcommand{\arraystretch}{1.2}
\begin{table*}[t]
\begin{center}
{\small
\begin{tabular}{ |c|c|c|c|c| }  
\cline{2-5}
\multicolumn{1}{c |}{} & Additional & Item-Level DP & User-Level DP & User-Level DP \\
\multicolumn{1}{c |}{} & Assumptions on $\ell$ & (Previous Work) & (Previous Work) & (Our Results) \\
\hline
& (no additional & & --- & \\
& assumption) &  $\tO_\eps\left(\frac{\sqrt{d}}{n}\right)$ & & $\tO_\eps\left(\frac{\sqrt{d}}{n\sqrt{m}}\right)$ for $n \geq \tOmega_\eps(1)$ \\
\cline{2-2} \cline{4-4}
& & \cite{BassilyST14} & $\tO_\eps\left(\frac{1}{\sqrt{n}} + \frac{\sqrt{d}}{n\sqrt{m}}\right)$ & (\Cref{thm:convex-erm}) \\
& Smooth & & for $n \geq \tOmega_\eps(1)$ & \\
& & & \cite{NarayananME22} &  \\
\cline{2-5}
ERM & Strongly Convex & & --- & \\
& & $\tO_\eps\left(\frac{d}{n^2}\right)$ & & $\tO_\eps\left(\frac{d}{n^2m}\right)$ for $n \geq \tOmega_\eps(1)$ \\
\cline{2-2} \cline{4-4}
& Strongly Convex & \cite{BassilyST14} & $\tO_\eps\left(\frac{d}{n^2m}\right)$ & (\Cref{thm:strongly-convex-erm}) \\
& \& Smooth & & for $n \geq \tOmega_\eps(1)$ &  \\
&  & & \cite{NarayananME22} & \\
\hline
& (no additional & & --- & \\
& assumption) &  $\tO_\eps\left(\frac{1}{\sqrt{n}} + \frac{\sqrt{d}}{n}\right)$ & & $\tO_\eps\left(\frac{1}{\sqrt{nm}} + \frac{\sqrt{d}}{n\sqrt{m}}\right)$ for $n \geq \tOmega_\eps(1)$ \\
\cline{2-2} \cline{4-4}
& & \cite{BassilyFTT19} & $\tO_\eps\left(\frac{1}{\sqrt{n m}} + \frac{\sqrt{d}}{n\sqrt{m}}\right)$ & (\Cref{thm:convex-sco}) \\
SCO & Smooth & & for $n \geq \tOmega_\eps(\min\{\sqrt[3]{m}, \sqrt{m}/d\})$ & \\
& & & \cite{NarayananME22} &  \\
\cline{2-5}
& Strongly Convex & $\tO_\eps\left(\frac{1}{n} + \frac{d}{n^2}\right)$ & --- & $\tO_\eps\left(\frac{1}{n m} + \frac{d}{n^2m}\right)$ for $n \geq \tOmega_\eps(1)$ \\
& & \cite{FeldmanKT20} & & (\Cref{thm:strongly-convex-sco}) \\
\hline
 \end{tabular}
 }
 \end{center}
 \caption{Summary of our results and previous results. In all rows, the loss function is assumed to be convex and Lipschitz. The $\tO_\eps$ hides polynomial dependency on the convexity, Lipschitzness, strong convexity and smoothness parameters, $\eps$, and polylogarithmic dependency on $1/\delta, n, m$. We remark that, while it seems plausible to derive bounds using their techniques,~\citet{LevySAKKMS21,NarayananME22} did not explicitly consider the strongly convex (and smooth) case for DP-SCO.} \label{tab:summary}
 \end{table*}

\section{Technical Overview}

Our main technical contribution is an \emph{improved output perturbation} algorithm for user-level DP compared to item-level DP. Recall that in item-level DP, the output perturbation algorithm~\cite{ChaudhuriMS11} computes the empirical risk minimizer $\theta^*$ and outputs $\theta^* + Z$ where $Z \sim \cN(\sigma^2 \cdot I)$ for a suitable $\sigma$; naturally, the smaller the $\sigma$ for which DP guarantees hold, the better the accuracy. It is known that for \emph{strongly convex} loss functions, this algorithm is DP for $\sigma = \tO_\eps\left(\frac{1}{n}\right)$. As discussed more below, we give a similar algorithm that only requires $\sigma = \tO_\eps(\frac{1}{n\sqrt{m}})$. This improvement is critical in our results. 

\textbf{Deletion Sensitivity.}
We exploit a refined notion of sensitivity to facilitate our improved output perturbation algorithm. Bounding the sensitivity of the quantity to be computed is one of the most used methods for achieving DP guarantees.  
Indeed, the DP guarantee of the output perturbation algorithm in \cite{ChaudhuriMS11} for item-level privacy follows from the fact that the ($\ell_2$-)sensitivity of the empirical risk minimizer is $O\left(\frac{1}{n}\right)$~\cite{Shalev-ShwartzSSS09}. Formally,
\begin{equation}\label{eq:introsensitivity}
 \|\theta^*(\bx) - \theta^*(\bx')\| \leq O(1/n),  
\end{equation}
for any two neighboring datasets $\bx, \bx'$.

Ideally, we would like the ``sensitivity'' of $\theta^*$ to become $\tO\left(\frac{1}{n\sqrt{m}}\right)$ for some notion of ``sensitivity''. However, the standard notion of sensitivity as above (or even local sensitivity~\cite{NissimRS07}) does not work: %
even for mean estimation\footnote{This corresponds to $\ell(\theta; x) = \|\theta - x\|^2$ (here $x \in \R^d$) for which the empirical risk minimizer $\theta^*(\bx)$ is the average over all the input points.}, we can change a user to have all their input vectors far from the mean, resulting in the same $O(1/n)$ sensitivity as before. Instead, we use the notion of \emph{deletion sensitivity}. Here, instead of considering $\bx'$ that results from {\em changing} a user's data in $\bx$, we only consider $\bx'$ that results from \emph{removing} a user's data completely.

\textbf{Bounding Deletion Sensitivity of Empirical Risk Minimizer.}
We show that the (local) deletion sensitivity of $\theta^*(\bx)$ is at most $\tO\left(\frac{1}{n\sqrt{m}}\right)$. To describe our approach, let us briefly recall the proof of \eqref{eq:introsensitivity} (item level setting, i.e., $m=1$) from \citet{Shalev-ShwartzSSS09}. The proof proceeds by bounding the norm of the gradient at $\theta^* := \theta^*(\bx)$ with respect to $\bx'$ (i.e., $\|\nabla \cL(\theta^*(\bx); \bx')\|$); strong convexity then implies that $\theta^*(\bx')$ is close to $\theta^*(\bx)$. The gradient norm bound is based on the observation that $\nabla \cL(\theta^*; \bx) = 0$ due to optimality, and that $\nabla \cL(\theta^*; \bx) - \nabla \cL(\theta^*; \bx')$ is only $1/n$ times a difference of the gradients of \emph{two input points} (that got changed from $\bx$ to $\bx'$). These two claims yield the desired $O(1/n)$ bound.

For the user-level setting, the situation is similar except that $\nabla \cL(\theta^*; \bx) - \nabla \cL(\theta^*; \bx')$ now becomes $O\left(\frac{1}{nm}\right)$ times the gradient of \emph{all input points of a single user} (that got removed from $\bx$ to $\bx'$). An observation we use here is that in SCO---where all $nm$ input points are drawn i.i.d.---we may view the input generation as a two-step process: (i) draw the $nm$ input points, and (ii) randomly allocate these $nm$ input points to $n$ users. With this view in mind, $\nabla \cL(\theta^*; \bx) = 0$ means that the sum of the $nm$ gradients is zero. The randomness in (ii) means that $\nabla \cL(\theta^*; \bx) - \nabla \cL(\theta^*; \bx')$ is now $O\left(\frac{1}{nm}\right)$ times the sum of $m$ vectors \emph{randomly chosen} from these $nm$ vectors that sum to zero. Applying concentration inequalities (and a union bound), we can show that w.h.p. $\|\nabla \cL(\theta^*; \bx) - \nabla \cL(\theta^*; \bx')\| \leq \tO\left(\frac{1}{n\sqrt{m}}\right)$ as desired.

\textbf{Noise Addition Algorithm for Deletion Sensitivity.}
Adding noise is still not trivial, even after bounding the (local) deletion sensitivity. As stated earlier, since we do not have the bound for the (standard) sensitivity, adding Gaussian noise directly to $\theta^*(\bx)$ will not yield the desired DP guarantee. To overcome this, we give an algorithm that adds noise w.r.t. the (local) deletion sensitivity. At a high-level, our algorithm has to perform a test to ensure that $\bx$ is ``sufficiently stable'' (akin to propose-test-release~\cite{DworkL09}) before adding Gaussian noise. Our algorithm is an adaptation of that of~\citet{KL21}, which focuses on the real-valued case and adds Laplace noise.

\textbf{From Output Perturbation to DP-SCO/DP-ERM.}
Finally, once we have the improved output perturbation algorithm, we use the localization-based algorithms (called Phased-SCO/Phased-ERM) of \citet{FeldmanKT20} with the enhanced output perturbation algorithm as subroutines to arrive at our results for DP-SCO/DP-ERM in the convex case. The strongly convex case follows from a known black-box reduction from~\citet{BassilyST14}. %

{\em Remark.} Our algorithm for ERM guarantees an $\tO_{\eps}\inparen{\frac{\sqrt{d}}{n\sqrt{m}}}$ excess risk w.h.p. over the input being a random permutation of any given dataset $\bx$. We emphasize that this is a mild assumption on the distribution of the dataset, and the same guarantees immediately follow for stronger assumptions such as the dataset $\bx$ being drawn from any exchangeable distribution e.g. drawn i.i.d. from $\cD$. Furthermore, we stress that it is impossible to have an excess risk bound for ERM that is better than $\tO_\eps(\sqrt{d}/n)$ for worst-case datasets since $x_{i,j}$ could be all the same for each $i$, which becomes essentially identical to the item-level setting with $m=1$.

\textbf{Comparison to Previous Work.}
Previous work~\cite{LevySAKKMS21,NarayananME22} on user-level DP-SCO and DP-ERM tackle the problem using privatized first order methods (i.e., variants of gradient descent), sometimes with regularization. The main tool in these works is a user-level DP algorithm for mean estimation of vectors, which is used to aggregate the gradients with errors smaller  than in the item-level setting. Such a result needs to rely on the fact that the average of the gradients of each user is well-concentrated; this can be interpreted as the average gradient having low deletion sensitivity. As discussed earlier, our result significantly generalizes this by showing that this also holds for the minimizer of any strongly convex function. Our algorithms also provide a novel paradigm of output perturbation for user-level DP learning---deviating from the first order methods explored in previous works.

In addition to the aforementioned work of~\citet{KL21}, a notion similar to local deletion sensitivity has been studied in the context of DP graph analysis under the names of ``empirical sensitivity''~\cite{ChenZ13} and ``down sensitivity''~\cite{RaskhodnikovaS16}. Several mechanisms were developed using this notion, including an algorithm for monotonic real-valued functions~\cite{ChenZ13} and for many graph parameters. However, we are not aware of a generic algorithm for the high-dimensional case similar to our \Cref{alg:output-perturbation}.

\section{Output Perturbation for Strongly Convex Losses}\label{sec:output-pert}

At the heart of our results is a new DP output perturbation algorithm (\Cref{alg:strongly-convex-output-perturbation}) for strongly convex losses. The guarantee of this algorithm does not hold for any worst-case dataset, but instead holds for a {\em random permutation} of any given dataset. In particular, for any permutation $\pi$ over $[n] \times [m]$, let $\bx^\pi$ be the permutation of $\bx$ by $\pi$, i.e., $x^\pi_{i,j} := x_{\pi(i,j)}$. As discussed in the previous section, this is a mild assumption but is required for our results.

\begin{theorem} \label{thm:output-pert}
Fix a $G$-Lipschitz and $\mu$-strongly convex loss $\ell$ and a sufficiently large constant $C$. For all $\eps, \delta, \beta > 0$ and $n \ge C \log(1/\delta) / \eps$, there exists 
an $(\eps, \delta)$-DP algorithm $\SCovOutputPert$, such that for all $\bx$, with probability $\ge 1 - \beta$ over a random permutation $\pi$ over $[n] \times [m]$, $\SCovOutputPert(\bx^\pi)$ outputs $\theta^*(\bx^\pi) + \cN(0, \sigma^2 \cdot I)$ where $\sigma = O\inparen{\frac{G\sqrt{\log n \log(1/\delta) / \eps + \log(1/\beta)}}{\mu n \sqrt{m}} \cdot \frac{(\log(1/\delta))^{1.5}}{\eps^2}}$.
\end{theorem}
The expected $\ell_2$-distance between the output estimate and the true minimizer thus scales as $\tO_\eps\inparen{\frac{\sqrt{d}}{n\sqrt{m}}}$. This should be compared with the item-level (i.e., $m = 1$) setting where the bound is $\tO_\eps(\sqrt{d}/n)$~\cite{ChaudhuriM08,ChaudhuriMS11}.

\subsection{Deletion Sensitivity \& A Generic Output Perturbation Algorithm}

Before we can prove \Cref{thm:output-pert}, we need to introduce the notion of local deletion sensitivity and present a generic output perturbation algorithm for low local deletion sensitivity functions and datasets. We stress that the algorithm in this section works for any such function and can be applied beyond the context of convex optimization.

\textbf{Local Deletion Sensitivity.}
For any $\bx = (\bx_1, \ldots, \bx_n)$, let $\bx_{-i}$, denote the dataset obtained by deleting the $i$th user's data $\bx_i$ from $\bx$. For any subset $S \subseteq [n]$, let $\bx_{-S}$ denote the dataset obtained by deleting %
$\bx_i$ from $\bx$ for all $i \in S$.

\begin{definition}\label{def:deletion-sensitivity}
The \emph{local ($\ell_2$-)deletion sensitivity} of function $f$ at dataset $\bx$ with $n$ users is defined as $\delsen f(\bx) := \max_{i \in [n]} \|f(\bx) - f(\bx_{-i})\|$.
Moreover, for $r \in \N$, let $\delsen_r f(\bx) := \max_{S \subseteq [n], |S| \le r} \delsen f(\bx_{-S})$.
\end{definition}

The difference between the usual definition of local sensitivity~\cite{NissimRS07} and that of local deletion sensitivity is that the latter definition only applies to {\em removal} of a user's data. This means that standard frameworks such as propose-test-release~\cite{DworkL09} cannot be directly used here. We however show that this sensitivity notion still allows us to design an algorithm with small error on any dataset for which $\delsen_r f(\bx)$ is small for any sufficiently large $r = \Theta(\log(1/\delta) / \eps)$. The  guarantee is given below.

\begin{theorem} \label{thm:output-pert-generic}
Let $f : \cX^{* \times m} \to \R^d$, and $\Delta > 0$ be a predefined parameter. There exists an $(\eps, \delta)$-DP algorithm that either outputs $\perp$ or a vector in $\R^d$. Furthermore, there exists $r = O(\log(1/\delta)/\eps)$ such that, on input dataset $\bx$ that satisfies $\delsen_{r} f(\bx) \leq \Delta$, it never returns $\perp$ and simply returns $f(\bx) + \cN(0, \sigma^2 \cdot I)$ where $\sigma = O\left(\Delta \cdot \frac{(\log(1/\delta))^{1.5}}{\eps^2}\right)$.
\end{theorem}

The general idea is to find a ``sufficiently stable'' dataset $\hbx$ and add noise to $f(\hbx)$. Although we may wish to just set $\hbx = \bx$ directly and check that the local sensitivity is small, we cannot do this, as changing a single datapoint can change whether we pass the test. Therefore, similar to the propose-test-release framework, we check for $\bx_{-S}$ for all subsets $S$ with $|S| \le R_1$ where $R_1$ is a shifted truncated discrete Laplace random variable, as defined below. This allows us to maintain the closeness of acceptance probability across neighboring input datasets. The full description is given in \Cref{alg:output-perturbation}. As stated earlier, our algorithm is a modification of that of~\citet{KL21}, which uses Laplace noise and a different distribution of $R_1$.

\begin{definition}[Shifted Truncated Discrete Laplace Distribution]
\label{def:tdlap}
For any $\eps, \delta > 0$, let $\kappa = \kappa(\eps, \delta) := 1 + \lceil \ln(1/\delta) / \eps \rceil$ and let $\TDLap(\eps, \delta)$ be the distribution supported on $\{0, \dots, 2\kappa\}$ with probability mass function at $x$ being proportional to $\exp\inparen{-\eps \cdot |x - \kappa|}$.
\end{definition}

\begin{algorithm}[ht]
\caption{$\DelOutputPert_{\eps, \delta, \Delta}(f; \bx)$}
\label{alg:output-perturbation}
\begin{algorithmic}[1]
\STATE \textbf{Input: } Dataset $\bx$, function $f : \cX^{* \times m} \to \R^d$
\STATE \textbf{Parameters: } Privacy parameters $\eps, \delta$; Target deletion sensitivity parameter $\Delta$
\STATE $\oeps \gets \frac{\eps}{2}$, $\odelta \gets \frac{\delta}{e^{\oeps} + 2}$, $\kappa \gets \kappa(\oeps, \odelta),\sigma \gets \frac{2\sqrt{\ln(2/\odelta)} \left(8\kappa \Delta\right)}{\oeps}$
\STATE Sample $R_1 \sim \TDLap(\oeps, \odelta)$
\COMMENT{\small See
\Cref{def:tdlap}}
\STATE $\cX^{R_1}_{\stable} \gets \set{\bx_{-S} : |S| \le R_1, \delsen_{4\kappa - |S|} f(\bx_{-S}) \le \Delta}$%
\IF{$|\cX^{R_1}_{\stable}| = \emptyset$}
\RETURN $\perp$
\ENDIF
\STATE Choose $\bx_{-S} \in \cX^{R_1}_{\stable}$ with smallest $|S|$ \\
\COMMENT{\small Ties broken arbitrarily} \label{line:choose-stable-set}
\RETURN $f(\bx_{-S}) + \cN(0, \sigma^2 \cdot I)$
\end{algorithmic}
\end{algorithm}

To prove \Cref{thm:output-pert-generic}, the following observation is useful.%

\begin{observation} \label{obs:stable-set-neighbors}
For neighboring datasets $\bx, \bx'$, and all $r_1 \in \Z_{\geq 0}$, if $\cX^{r_1}_{\stable}(\bx') \ne \emptyset$, then $\cX^{r_1 + 1}_{\stable}(\bx) \ne \emptyset$.
\end{observation}
\begin{proof}
Suppose $\bx' = (\bx'_1, \ldots, \bx'_n)$. Let $\bx'_{-S'}$ be an element of $\cX^{r_1}_{\stable}(\bx')$. That is, we have $|S'| \le r_1$ and $\delsen_{r_2}f(\bx'_{-S'}) \le \Delta$ for $r_2 = 4\kappa - r_1$.
Let $i \in [n]$ denote the user on which $\bx$ and $\bx'$ differ. We consider two cases:
If $i \in S'$, then we simply have $\bx_{-S'} = \bx'_{-S'}$ and therefore $\bx_{-S'}$ also belongs  to $\cX^{r_1+1}_{\stable}(\bx)$.
If $i \notin S'$, let $S = S' \cup \{i\}$. We have $|S| \leq r_1 + 1$ and $\delsen_{r_2 - 1}(\bx_{-S}) \le \delsen_{r_2}(\bx'_{-S'})$. This means that $\bx_{-S} \in \cX^{r_1+1}_{\stable}(\bx)$ as well.
\end{proof}

\begin{proof}[Proof of \Cref{thm:output-pert-generic}]
Let $\cA$ be the $\DelOutputPert$ algorithm (\Cref{alg:output-perturbation}).

\textbf{Privacy Analysis.}
Let $\bx, \bx'$ be neighboring datasets. First,%
\begin{align}
&\textstyle\Pr[\cA(\bx) = \perp] \nonumber \\
&\textstyle= \sum_{r_1=0}^{2\kappa} \ind[\cA(\bx) = \perp \mid R_1 = r_1] \cdot \Pr[R_1 = r_1] \nonumber \\
&\textstyle= \sum_{r_1=0}^{2\kappa} \ind[\cX^{r_1}_{\stable}(\bx) = \emptyset] \cdot \Pr[R_1 = r_1] \nonumber \\
&\textstyle\leq \Pr[R_1 = 0] + \sum_{r_1=1}^{2\kappa} \ind[\cX^{r_1}_{\stable}(\bx) = \emptyset] \cdot \Pr[R_1 = r_1] \nonumber \\
&\textstyle\leq \odelta + \sum_{r_1=1}^{2\kappa} \ind[\cX^{r_1}_{\stable}(\bx) = \emptyset] \cdot e^{\oeps} \cdot \Pr[R_1 = r_1 - 1] \nonumber \\
&\textstyle\leq \odelta + \sum_{r_1=1}^{2\kappa} \ind[\cX^{r_1 - 1}_{\stable}(\bx') = \emptyset] \cdot e^{\oeps} \cdot \Pr[R_1 = r_1 - 1] \nonumber \\
&\textstyle\leq \odelta + e^{\oeps} \cdot \sum_{r_1=0}^{2\kappa} \ind[\cA(\bx') = \perp \mid R_1 = r_1] \cdot \Pr[R_1 = r_1] \nonumber \\
&\textstyle= \odelta + e^{\oeps} \cdot \Pr[\cA(\bx') = \perp]. \label{eq:perp-dp-bound}
\end{align}

Next, consider any set $S_0 \subseteq \R^d$. We have
\begin{align}
&\textstyle\Pr[\cA(\bx) \in S_0]  \nonumber \\ 
&\textstyle \leq \sum_{r_1=0}^{2\kappa} \Pr[\cA(\bx) \in S_0 \mid R_1 = r_1] \cdot \Pr[R_1 = r_1] \nonumber \\
&\textstyle= \sum_{r_1=0}^{2\kappa - 1} \Pr[\cA(\bx) \in S_0 \mid R_1 = r_1] \cdot \Pr[R_1 = r_1]  \nonumber \\
&\textstyle \qquad ~+~ \Pr[R_1 = 2\kappa] \nonumber\\
&\textstyle\leq \odelta + e^{\oeps} \cdot \sum_{r_1=0}^{2\kappa - 1} \Pr[\cA(\bx) \in S_0 \mid R_1 = r_1] \nonumber \\
&\textstyle \qquad\qquad\qquad\qquad\qquad \cdot \Pr[R_1 = r_1 + 1],
\label{eq:expand-real-output-set}
\end{align}
where the last inequality follows
since $R_1 \sim \TDLap(\oeps, \odelta)$.
To bound the term $\Pr[\cA(\bx) \in S_0 \mid R_1 = r_1]$ for $r_1 < 2\kappa$, observe that if it is non-zero, then it must be that $\cA(\bx) \ne \perp$ or equivalently that $\cX^{r_1}_{\stable}(\bx) \ne \emptyset$; \Cref{obs:stable-set-neighbors} then implies that $\cX^{r_1 + 1}_{\stable}(\bx) \ne \emptyset$, or equivalently, $\cA(\bx') \ne \perp$. Let $\hbx_{r_1}$ and $\hbx'_{r_1 + 1}$ be the sets chosen on Line~\ref{line:choose-stable-set} when we run the algorithm on input $\hbx, R_1 = r_1$ and $\hbx', R_1 = r_1 + 1$ respectively. Let $\bx^* := \hbx_{r_1} \cap \hbx'_{r_1 + 1}$. We have $|\bx^*| \geq |\bx| - r_1 - (r_1 + 1) \geq |\bx| - 4\kappa$. Therefore, since $\hbx_{r_1} \in \cX^{r_1}_{\stable}(\bx)$ and $\hbx'_{r_1+1} \in \cX^{r_1 + 1}_{\stable}(\bx')$, we must have
\begin{align*}
\|f(\hbx_{r_1}) - f(\bx^*)\| \leq \Delta \cdot |\hbx - \hbx'| \leq 4\kappa \cdot \Delta,
\end{align*}
and similarly,
\begin{align*}
\|f(\hbx'_{r_1+1}) - f(\bx^*)\| \leq \Delta \cdot |\hbx - \hbx'| \leq 4\kappa \cdot \Delta.
\end{align*}
Combining the two, we can conclude that
\begin{align*}
\|f(\hbx_{r_1}) - f(\hbx'_{r_1+1})\| \leq 8\kappa \cdot \Delta.
\end{align*}
From the privacy guarantee of the Gaussian mechanism (e.g., \citet[][Appendix A]{DworkR14}), we have
\begin{align*}
\lefteqn{\Pr[f(\hbx_{r_1}) + \cN(\sigma^2 \cdot I) \in S_0]} \\
& \leq e^{\oeps} \cdot \Pr[f(\hbx'_{r_1+1}) + \cN(\sigma^2 \cdot I) \in S_0] + \odelta.
\end{align*}
Note that $\Pr[f(\hbx_{r_1}) + \cN(\sigma^2 \cdot I) \in S_0] = \Pr[\cA(\bx) \in S_0 \mid R_1 = r_1]$, while $\Pr[f(\hbx'_{r_1+1}) + \cN(\sigma^2 \cdot I) \in S_0] = \Pr[\cA(\bx') \in S_0 \mid R_1 = r_1+1]$.

Plugging this back to \eqref{eq:expand-real-output-set}, we get
\begin{align}
&\textstyle\Pr[\cA(\bx) \in S_0] \nonumber \\
&\textstyle \leq \odelta + e^{\oeps} \cdot \sum_{r_1=0}^{2\kappa - 1} \Pr[\cA(\bx) \in S_0 \mid R_1 = r_1] \nonumber \\
&\textstyle \qquad\qquad\qquad\qquad\qquad  \cdot \Pr[R_1 = r_1 + 1] \nonumber \\
&\textstyle= \odelta + e^{\oeps} \cdot \sum_{r_1=0}^{2\kappa - 1} \left(e^{\oeps} \cdot \Pr[\cA(\bx') \in S_0 \mid R_1 = r_1 + 1] + \odelta\right) \nonumber \\
&\textstyle \qquad\qquad\qquad\qquad\qquad \cdot \Pr[R_1 = r_1 + 1] \nonumber \\
&\textstyle\leq (e^{\oeps} + 1) \odelta + e^{2\oeps} \cdot \sum_{r_1=0}^{2\kappa} \Pr[\cA(\bx') \in S_0 \mid R_1 = r_1] \nonumber \\
&\textstyle \qquad\qquad\qquad\qquad\qquad \cdot \Pr[R_1 = r_1] \nonumber \\
&\textstyle= (e^{\oeps} + 1) \odelta + e^{2\oeps} \cdot \Pr[\bA(\bx') \in S_0]. \label{eq:real-output-dp-bound}
\end{align}

Now, consider any set $S$ of outcomes. Let $S_0 = S \cap \R^d$ and $S_{\perp} = S \cap \{\perp\}$. Then, we have
\begin{align*}
&\Pr[\cA(\bx) \in S]
= \Pr[\cA(\bx) \in S_0] + \Pr[\cA(\bx) \in S_{\perp}] \\
&\overset{\eqref{eq:real-output-dp-bound},\eqref{eq:perp-dp-bound}}{\leq}  \left((e^{\oeps} + 1) \odelta + e^{2\oeps} \cdot \Pr[\bA(\bx') \in S_0]\right) \\
& \qquad\qquad + \left(\odelta + e^{\oeps} \cdot \Pr[\cA(\bx') = S_\perp]\right) \\
&\leq (e^{\oeps} + 2)\odelta + e^{2\oeps} \Pr[\cA(\bx') \in S] \\
&\leq \delta + e^{\eps} \cdot \Pr[\cA(\bx') \in S].
\end{align*}
Therefore, the algorithm is $(\eps, \delta)$-DP as desired.

\textbf{Accuracy Analysis.} Let $\bx$ be any dataset such that $\delsen_{4\kappa} \bx \leq \Delta$. This means that, for any $0 \leq R_1 \leq 4\kappa$, $\delsen_{4\kappa - R_1} \bx \leq \Delta$. In other words, $\bx$ belongs to $\cX^{R_1}_{\stable}$. Thus, we always have $\hbx = \bx$ and the output is simply drawn from $f(\bx) + \cN(0, \sigma^2 \cdot I)$ as claimed.
\end{proof}

\subsection{Deletion Sensitivity of Optimizers of Strongly Convex Losses}

Having provided a generic noising algorithm for functions with low local deletion sensitivity, the next step is to show that the function that we care about for convex optimization---the empirical risk minimizer---has low deletion sensitivity (with high probability), as formalized below.

\begin{theorem} \label{thm:deletion-stability}
Let $\ell$ be any $\mu$-strongly convex loss function such that $\|\nabla \ell(\theta; x)\| \leq G$ for all $\theta \in \cK, x \in \cX$. For all $\bx \in \cX^{n \times m}$ and $\beta < 1/e$, with probability $1 - \beta$ over the choice of a random permutation $\pi$ over $[n] \times [m]$, we have
\begin{align*}
\textstyle\|\theta^*(\bx^\pi) - \theta^*(\bx^\pi_{-n})\| ~\leq~ \frac{5G\sqrt{\log(1/\beta)}}{\mu (n - 1) \sqrt{m}}.
\end{align*}
\end{theorem}
Before proving this, we note that by applying a union bound over all the $n$ users and all subsets $S$ of size at most $r$, we arrive at \Cref{cor:delsen-neighborhood-whp}. \Cref{thm:output-pert} now follows by defining $\SCovOutputPert$ (\Cref{alg:strongly-convex-output-perturbation}) that invokes $\DelOutputPert$ on the function $f$ being the empirical loss, and combining \Cref{cor:delsen-neighborhood-whp} with \Cref{thm:output-pert-generic}  (setting $r = 4\kappa$).

\begin{algorithm}[ht]
\caption{$\SCovOutputPert_{\eps, \delta, \beta, G, \mu, \cK}(\ell; \bx)$}
\label{alg:strongly-convex-output-perturbation}
\begin{algorithmic}[1]
\STATE \textbf{Input:} Dataset $\bx$, loss function $\ell : \cK \times \cX \to \R$
\STATE \textbf{Parameters:} Privacy parameters $\eps, \delta$; Target failure probability $\beta$; Lipschitz parameter $G$; Strong convexity $\mu$
\STATE $\Delta \gets \frac{10G\sqrt{\log(1/\beta)}}{\mu n \sqrt{m}}$
\RETURN $\DelOutputPert_{\eps, \delta, \Delta}(f; \bx)$,\\
where $f(\cdot) := \argmin_{\theta} \cL(\theta; \cdot)$
\end{algorithmic}
\end{algorithm}

\begin{corollary} \label{cor:delsen-neighborhood-whp}
Let $\ell$ be any $G$-Lipschitz, $\mu$-strongly convex loss. For all $\bx \in \cX^{n \times m}$ and $r \leq n/2$, with probability $1 - \beta$ over the choice of a random permutation $\pi$ over $[n] \times [m]$, we have
\begin{align*}
\textstyle\delsen_r \theta^*(\bx^\pi) ~\leq~ O\inparen{\frac{G\sqrt{r\log n + \log(1/\beta)}}{\mu n \sqrt{m}}}.
\end{align*}
\end{corollary}

\noindent In order to prove \Cref{thm:deletion-stability}, we use the following lemma, proved in \Cref{apx:proof_vector_concentration}.

\begin{lemma}\label{lem:vector_concentration}
Let $\bv_1, \ldots, \bv_N \in \R^d$ be any set of vectors satisfying $\sum_i \bv_i = \bzero$ and $\|\bv_i\| \le G$ for all $i$. For all $\beta < 1/e$, over choice of a random permutation $\pi$ over $[N]$, it holds that
\begin{align*}
    \textstyle\Pr\left[\left \| \sum_{j \in [m]} \bv_{\pi(j)} \right \| ~>~ 5G \sqrt{m \log(1/\beta)}\right] ~\le~ \beta.
\end{align*}
\end{lemma}

\begin{proof}[Proof of \Cref{thm:deletion-stability}]
Let $\theta^* := \theta^*(\bx)$; note that due to the symmetric nature of $\cL(\theta; \cdot)$, it holds that $\theta^*(\bx)=\theta^*(\bx^{\pi})$ for all permutations $\pi$. Let $\theta^{*,\pi}_{-n} := \theta^*(\bx^{\pi}_{-n})$.
Since  $\cL(\cdot; \bx^{\pi}_{-n})$ is $\mu$-strongly convex\footnote{$f$ is \emph{$\mu$-strongly convex} iff $\|\nabla f(\theta) - \nabla f(\theta')\| \ge \mu \|\theta - \theta'\|$.}, we have that
\begin{align}
\|\nabla \cL(\theta^*; \bx^{\pi}_{-n}) - \nabla \cL(\theta^{*,\pi}_{-n}; \bx^{\pi}_{-n})\| ~\geq~ \mu \|\theta^* - \theta^{*,\pi}_{-n}\|.\label{eq:strong-convexity}
\end{align}
Next, we upper bound $\|\nabla \cL(\theta^*; \bx^{\pi}_{-n})\|$.
\begin{align*}
	0 \textstyle= \nabla \cL(\theta^*; \bx^{\pi}) \textstyle= \frac{n-1}{n} \cdot \nabla \cL(\theta^*; \bx^{\pi}_{-n}) + \frac{1}{n} \cdot \nabla \cL(\theta^*; \bx^{\pi}_n),
\end{align*}
and hence
\begin{align}
	&\textstyle \left\|\nabla \cL(\theta^*; \bx^{\pi}_{-n})\right\| ~=~ \frac{1}{n-1}\left\|\nabla \cL(\theta^*; \bx^{\pi}_{n})\right\| \nonumber \\
	&\textstyle=~ \left\|\frac{1}{(n - 1)m} \sum_{j \in [m]} \nabla \ell(\theta; x_{\pi(n, j)})\right\|. \label{eq:grad-expand}
\end{align}
Since $\sum_{i \in [n],j \in [m]} \nabla \ell(\theta; x_{i, j}) = 0$ and $\|\nabla \ell(\theta; x_{i,j})\| \le G$, we have from \Cref{lem:vector_concentration} that
\begin{align} \label{eq:concen-grad-sum}
	\textstyle\Pr\left[\left\|\sum_{j \in [m]} \nabla \ell(\theta; x_{\pi(n, j)})\right\| \leq 5G\sqrt{m\log(1/\beta)}\right] \nonumber \\
	\geq 1 - \beta.
\end{align}
Putting~\eqref{eq:grad-expand} and \eqref{eq:concen-grad-sum} together, we have that,
\begin{align*} \textstyle
	\Pr\left[\|\nabla \cL(\theta^*; \bx^{\pi}_{-n}) \|
	\leq \frac{5G\sqrt{\log \frac{1}{\beta}}}{(n - 1)\sqrt{m}}\right] \ge 1-\beta.
\end{align*}
Combining this with \eqref{eq:strong-convexity}, and noting that  $\nabla \cL(\theta^{*,\pi}_{-n}; \bx^{\pi}_{-n}) = 0$ since $\theta^{*,\pi}_{-n}$ is the minimizer of $\cL(\cdot; \bx^{\pi}_{-n})$, completes the proof.
\end{proof}

\section{User-Level DP-ERM}\label{sec:erm}

In this section, we describe our algorithms for DP-ERM and prove their excess risk bounds. As in \Cref{sec:output-pert}, our guarantee holds for a random permutation of any dataset---a mild but necessary assumption.

\subsection{Convex Losses}

The formal guarantee when the loss is only assumed to be convex (and Lipschitz) is given below. 

\begin{theorem} \label{thm:convex-erm}
For any $G$-Lipschitz loss $\ell$, there exists an $(\eps, \delta)$-DP mechanism that, for all $n \geq \tOmega\inparen{\frac{\log(1/\delta)\log(m)}{\eps}}$, outputs $\htheta$ such that %
\begin{align*}
\textstyle \E_{\pi, \htheta \gets \cM(\bx^{\pi})}[\cL(\htheta; \bx^\pi)] - \cL(\theta^*; \bx^\pi) ~\leq~ \tO_\eps\inparen{\frac{RG\sqrt{d}}{n\sqrt{m}}},
\end{align*}
where $\tO_{\eps}$ hides a multiplicative factor of $\poly(\log(1/\delta), \log(nm), 1/\eps)$.
\end{theorem}

We use Phased-ERM algorithm similar to~\citet{FeldmanKT20}, which requires solving a regularized ERM in each step. Our proof below closely follows their proofs, although we change the algorithm slightly because their proof is for SCO whereas the analysis below is for ERM. In particular, for ERM, we need to use the full dataset in each step. We also change some parameters accordingly. The full description is in \Cref{alg:dp-phased-erm}; note that on line \ref{line:output-pert}, we only optimize over the set $\cK_i$, and use Lipschitz constant $2G$ and strong convexity parameter $\lambda_i$.

\begin{algorithm}[ht]
\caption{$\PhasedERM$.}
\label{alg:dp-phased-erm}
\begin{algorithmic}[1]
\STATE \textbf{Input: } Dataset $\bx$, loss function $\ell: \cK \times \cX \to \R$ that is convex and $G$-Lipschitz
\STATE \textbf{Parameters: } Privacy parameters $\eps, \delta$; Regularizer coefficient $\lambda$; Target failure probability $\beta$
\STATE $T \gets \lceil \log(n m) \rceil$ \COMMENT{\small Number of iterations}
\STATE $\eps' \gets \eps/T, \delta' \gets \delta/T$ \COMMENT{\small Per-iteration privacy budgets}
\STATE $\beta' \gets \beta/T$ \COMMENT{\small Per-iteration failure probability}
\STATE $\htheta_0 \gets$ arbitrary element of $\cK$ \COMMENT{\small Initial parameter}
\FOR{$i = 1, \dots, T$}
\STATE $\lambda_i \gets \lambda \cdot 4^i$
\STATE $R_i \gets G / \lambda_i$
\STATE Let $\ell_i(\theta; x) := \ell(\theta; x) + \frac{\lambda_i}{2} \cdot \|\theta - \htheta_{i - 1}\|^2$%
\STATE $\cK_i \gets \cK \cap \cB_d\inparen{\theta_{i-1}, R_i}$ %
\STATE $\htheta_i \gets \SCovOutputPert_{\eps', \delta', \beta', 2G, \lambda_i, \cK_i}(\ell_i; \bx)$%
\label{line:output-pert}
\ENDFOR
\RETURN $\htheta_T$
\end{algorithmic}
\end{algorithm}  

To analyze the accuracy, let $\theta^*_i := \theta^*_{\ell_i, \cK_i}(\bx)$ for all $i \in [T]$. It should be noted that $\theta^*_i$ is also equal to $\theta^*_{\ell_i, \cK}(\bx)$ (where the optimization is over $\cK$ instead of $\cK_i$). Furthermore, within $\cK_i$, the loss $\cL_i$ is $2G$-Lipschitz. 
We start with the following lemma, which is an analogue of \citet[][Lemma 4.7]{FeldmanKT20}.
\begin{lemma} \label{lem:phased-erm-intermediate}
For any $\theta \in \cK$ and $i \in [T]$, we have
\begin{align*}
\textstyle \cL(\theta^*_i; \bx) - \cL(\theta; \bx) \leq \frac{\lambda_i}{2} \cdot \|\htheta_{i - 1} - \theta\|^2.
\end{align*}
\end{lemma}
\begin{proof}
This is simply because
\begin{align*}
\textstyle\cL(\theta^*_i; \bx) - \cL(\theta; \bx) 
\leq \cL_i(\theta^*_i; \bx) - \cL(\theta; \bx) & \\
\textstyle\leq \cL_i(\theta; \bx) - \cL(\theta; \bx) 
= \frac{\lambda_i}{2} \cdot \|\theta - \htheta_{i - 1}\|^2.  &
\qedhere
\end{align*}
\end{proof}
We are now ready to prove \Cref{thm:convex-erm}. The usual analysis of the ``standard'' $\PhasedERM$ algorithm in \citet{FeldmanKT20}---where $\SCovOutputPert$ is replaced by an algorithm that just adds Gaussian noise to the true optimizer---shows that it has small excess risk. We then relate \Cref{alg:dp-phased-erm} back to this ``standard'' version by using \Cref{thm:output-pert} to show that the output distribution of our algorithm (over random $\pi$) is very close in total variation distance to this standard version. This idea is formalized below.

\begin{proof}[Proof of \Cref{thm:convex-erm}]
We run the $\PhasedERM$ algorithm (\Cref{alg:dp-phased-erm}) where we set $\lambda = \frac{G\sqrt{d}}{Rn\sqrt{m}}$ and $\beta = \frac{1}{nm}$; throughout the proof, we will use $\cM$ as a shorthand for this algorithm. The privacy guarantee follows immediately from the fact that each call to $\SCovOutputPert$ is $(\eps',\delta')$-DP and the basic composition of DP.

To understand its accuracy guarantee, let us start by considering another algorithm $\cM'$ that is the same as $\cM$ except that on line \ref{line:output-pert} we do not call $\SCovOutputPert$ but instead directly let $\htheta_i \gets \theta^*_i(\bx) + \cN(0, \sigma_i^2 \cdot I)$ where $\sigma_i := \sigma(\eps',\delta',\beta', 2G, \lambda_i)$ is as in \Cref{thm:output-pert}.

For convenience, we let $\theta_0 = \theta^*(\bx)$. We have
\begin{align*}
&\cL(\htheta_T; \bx) - \cL(\theta^*_0; \bx) \\
&= (\cL(\htheta_T; \bx) - \cL(\theta^*_T; \bx)) \\
&\textstyle \qquad+ \sum_{i=1}^T \inparen{\cL(\theta^*_i; \bx) - \cL(\theta^*_{i - 1}; \bx)} \\
&\textstyle\leq G \cdot \|\htheta_T - \theta^*_T\| + \sum_{i=1}^T \frac{\lambda_i}{2} \cdot \|\htheta_{i - 1} - \theta^*_{i - 1}\|^2\\
&\textstyle\leq O\inparen{\lambda R^2} + G \cdot \|\htheta_T - \theta^*_T\| +  \sum_{i=1}^{T-1} \frac{\lambda_{i+1}}{2} \cdot \|\htheta_{i} - \theta^*_{i}\|^2, 
\end{align*}
where we used \Cref{lem:phased-erm-intermediate} for the first inequality.

Thus, we have (where the expectation is over the randomness of $\cM'$, i.e., noise drawn from $\cN(0, \sigma_i^2 \cdot I)$ for each $i \in [T]$)
\begin{align*}
&\textstyle\E_{\htheta \gets \cM'(\bx)}[\cL(\htheta; \bx)] - \cL(\theta^*; \bx) \\
&\textstyle\leq O\inparen{\lambda R^2 + G\sqrt{d} \cdot \sigma_T + \sum_{i=1}^{T-1} \lambda_{i+1} \cdot d \cdot \sigma_i^2} \\
&\textstyle\leq \tO_\eps\inparen{\lambda R^2 +  \frac{G^2 \sqrt{d}}{\lambda_T \cdot n\sqrt{m}} + \sum_{i=1}^{T-1} \frac{\lambda_{i+1} \cdot d G^2}{\lambda_i^2 \cdot n^2 m}} \\
&\textstyle\leq \tO_\eps\inparen{\lambda R^2 + \frac{d G^2}{\lambda \cdot n^2 m}},
\end{align*}
where the last inequality comes from our setting $\lambda_i = \lambda \cdot 4^i$. Finally, from our setting of $\lambda = \frac{G\sqrt{d}}{Rn\sqrt{m}}$, we can conclude 
\begin{align} \label{eq:idealized-small-excess-risk}
\textstyle\E_{\htheta \gets \cM'(\bx)}[\cL(\htheta; \bx)] - \cL(\theta^*; \bx) \leq \tO_\eps\inparen{\frac{RG\sqrt{d}}{n\sqrt{m}}}. 
\end{align}

Let $P'$ denote the distribution of the output of running $\cM'$ on $\bx$, and let $P$ denote the distribution of the output of running $\cM$ on $\bx^{\pi}$ where $\pi$ is a uniformly random permutation. Next, we will show that
\begin{align} \label{eq:diseq-whp}
\textstyle \dTV(P', P) \leq \beta.
\end{align}
Before proving this, note that \eqref{eq:idealized-small-excess-risk} and \eqref{eq:diseq-whp} together imply the bound in the theorem because we have
\begin{align*}
&\textstyle \E_{\pi, \htheta \gets \cM(\bx^{\pi})}[\cL(\htheta; \bx)] - \cL(\theta^*; \bx) \\
&\overset{\eqref{eq:diseq-whp}}{\leq} \textstyle \E_{\htheta \gets \cM'(\bx)}[\cL(\htheta; \bx)] - \cL(\theta^*; \bx) + \beta \cdot RG \\
&\overset{\eqref{eq:idealized-small-excess-risk}}{\leq} \textstyle \tO_\eps\inparen{\frac{RG\sqrt{d}}{n\sqrt{m}}}.
\end{align*}

We are left with proving \eqref{eq:diseq-whp}. To do this, for every $i \in \{0, \dots, T\}$, consider a hybrid algorithm $\cM_i$ where, in the first $i$ iterations, we follow $\cM'$ and, in the remaining iterations, we follow $\cM$. Let $P_i$ denote the probability distribution of the output of $\cM_i$ on input $\bx^{\pi}$ where $\pi$ is a uniformly random permutation. Notice that $P_0 = P$ and $P_T = P'$. 

For every $i \in [T]$, consider $P_i$ and $P_{i - 1}$. They differ only in the $i$th iteration. Thus, $\dTV(P_i, P_{i - 1})$ is at most the probability that $\SCovOutputPert$ does not output a sample from $\theta^*_i + \cN(0, \sigma_i^2 \cdot I)$. By \Cref{thm:output-pert}, the probability (over $\pi$) that this happens is at most\footnote{Note that the distribution of $\theta^*_i$ is independent of $\pi$ since we are running $\cM'$ for the first $i - 1$ steps. Thereby, we can still apply \Cref{thm:output-pert}, which only relies on the randomness in $\pi$.} $\beta'$. Therefore, we have that $\dTV(P_i, P_{i - 1}) \leq \beta'$. 

Thus, $\dTV(P, P') \leq \sum_{i \in [T]} \dTV(P_{i- 1}, P_i) \leq T \cdot \beta' = \beta$, concluding our proof.
\end{proof}

\subsection{Strongly Convex Losses}

For strongly convex losses, we can get an improved bound:

\begin{theorem}\label{thm:strongly-convex-erm}
For any $G$-Lipschitz, $\mu$-strongly convex loss $\ell$, there exists an $(\eps, \delta)$-DP mechanism that, for all $n \geq \tOmega\inparen{\frac{\log(1/\delta)\log(m)}{\eps}}$, outputs $\htheta$ such that %
\begin{align*}
\textstyle \E_{\pi, \htheta \gets \cM(\bx^{\pi})}[\cL(\htheta; \bx^\pi)] - \cL(\theta^*; \bx^\pi) ~\leq~ \tO_\eps\inparen{\frac{G^2d}{\mu n^2 m}},
\end{align*}
where $\tO_{\eps}$ hides a multiplicative factor of $\poly(\log(1/\delta), \log(nm), 1/\eps)$.
\end{theorem}

We obtain the above result by reducing back to the convex case. This reduction essentially dates back to \citet{BassilyST14} and works as follows: first we apply the output perturbation algorithm (\Cref{thm:output-pert}). With high probability, the output is within a ball of radius $R = \tO_\eps\left(\frac{G\sqrt{d}}{\mu n\sqrt{m}}\right)$. We then run \Cref{thm:convex-erm} using this $R$, which yields the final excess risk of $\tO_\eps\inparen{\frac{G\sqrt{d}}{\mu n\sqrt{m}} \cdot \frac{G\sqrt{d}}{n\sqrt{m}}} = \tO_\eps\inparen{\frac{G^2d}{\mu n^2 m}}$ as desired.
The full proof is presented in \Cref{subsec:strongly-convex-erm}.

\section{User-Level DP-SCO}

We next describe our algorithms for DP-SCO and their excess risk guarantees.

\subsection{Convex Losses}

For the convex (and Lipschitz) loss case, the risk bound is similar to that of \Cref{thm:convex-erm} except with an additional additive term $O(1/\sqrt{nm})$:

\begin{theorem} \label{thm:convex-sco}
For any $G$-Lipschitz convex loss $\ell$, there exists an $(\eps, \delta)$-DP mechanism that, for all $n \geq \tO\inparen{\frac{\log(1/\delta)\log(m)}{\eps}}$, outputs $\htheta$ such that %
\begin{align*}
&\textstyle \E_{\bx \sim \cD^{n \times m}, \htheta \gets \cM(\bx)}[\cL(\htheta; \cD)] - \cL(\theta^*; \cD) \\
&\textstyle ~\leq~~ \tO_\eps\inparen{RG \inparen{\frac{\sqrt{d}}{n\sqrt{m}} + \frac{1}{\sqrt{nm}}}},
\end{align*}
where $\tO_{\eps}$ hides a multiplicative factor of $\poly(\log(1/\delta), \log(nm), 1/\eps)$.
\end{theorem}

The arguments in the previous section also extend to SCO. The idea as before is to replace the output perturbation step in Algorithm 3 of \citet{FeldmanKT20} with $\SCovOutputPert$. The full algorithm is presented in \Cref{alg:dp-phased-sco}; note that in the $i$th iteration, we only use the input from users $2^{-i}n, \dots, 2^{-i+1}n$ to perform $\SCovOutputPert$.

\begin{algorithm}[ht]
\caption{$\PhasedSCO$}
\label{alg:dp-phased-sco}
\begin{algorithmic}[1]
\STATE \textbf{Input: } Dataset $\bx$, loss function $\ell: \Theta \times \cX \to \R$ that is convex and $G$-Lipschitz
\STATE \textbf{Parameters: } Privacy parameters $\eps, \delta$; Regularizer coefficient $\lambda$; Target failure probability $\beta$
\STATE $N_0 = C \log(1/\delta)/\eps$ for some sufficiently large constant $C$
\STATE $T \gets \lceil \log(n/N_0) \rceil$ \COMMENT{\small Number of iterations}
\STATE $\beta' = \beta/T$ \COMMENT{\small Per-iteration failure probability}
\STATE $\htheta_0 \gets$ arbitrary element of $\cK$ \COMMENT{\small Initial parameter}
\FOR{$i = 1, \dots, T$}
\STATE $\lambda_i = \lambda \cdot 4^i$
\STATE $R_i = G / \lambda_i$
\STATE Let $\ell_i(\theta; x) := \ell(\theta; x) + \frac{\lambda_i}{2} \cdot \|\theta - \htheta_{i - 1}\|^2$%
\STATE $\cK_i \gets \cK \cap \cB_d\inparen{\theta_{i-1}, R_i}$ %
\STATE $\bx^{(i)} = (\bx_{\ell}: \ell \in [2^{-i} n, 2^{-i+1} n])$
\STATE $\htheta_i \gets \SCovOutputPert_{\eps, \delta, \beta', 2G, \lambda_i, \cK_i}(\ell_i; \bx^{(i)})$%
\label{line:output-pert-sco}
\ENDFOR
\RETURN $\htheta_T$
\end{algorithmic}
\end{algorithm}  

Similar to before, let $\theta^*_i := \theta^*_{\ell_i, \cK_i}(\cD)$ for all $i \in [T]$. Again, note that $\theta^*_i = \theta^*_{\ell_i, \cK}(\cD)$ (where the optimization is over $\cK$ instead of $\cK_i$). Furthermore, within $\cK_i$, the loss $\cL_i$ is $2G$-Lipschitz. 

We use the following lemma, analogous to \Cref{lem:phased-erm-intermediate}.
\begin{lemma} \label{lem:phased-sco-intermediate}
For any $\theta \in \cK$ and $i \in [T]$, we have
\begin{align*}
\textstyle \cL(\theta^*_i; \cD) - \cL(\theta; \cD) \leq \frac{\lambda_i}{2} \cdot \|\htheta_{i - 1} - \theta\|^2 + \frac{16 G^2}{\lambda_i n_i}.
\end{align*}
\end{lemma}
\begin{proof}
The objective $\ell_i(\theta;\bx^{(i)})$ is $(2G)$-Lipschitz and is $\lambda_i$-strongly convex. Therefore, by \citet[][Theorem 7]{Shalev-ShwartzSSS09}, we get that 
\[
\textstyle \cL(\theta^*_i;\cD) - \cL(\theta;\cD) \leq \frac{\lambda_i}{2} \cdot \|\theta - \htheta_{i - 1}\|^2 + \frac{4(2G)^2}{\lambda_i n_i}.\qedhere
\]
\end{proof}

\begin{proof}[Proof of \Cref{thm:convex-sco}]
We run the $\PhasedSCO$ algorithm (\Cref{alg:dp-phased-sco}) where we set $\lambda = \frac{G\sqrt{d}}{Rn\sqrt{m}}$ and $\beta = \frac{1}{nm}$; we will use $\cM$ as a shorthand for this algorithm. The main difference from \Cref{alg:dp-phased-erm} is that we use different batch sizes (and do not reuse sample points across phases). The analysis is similar to the proof of \Cref{thm:convex-erm} with corresponding changes to \Cref{lem:phased-erm-intermediate}. 
The privacy guarantee follows immediately from the fact that each call to $\SCovOutputPert$ is $(\eps,\delta)$-DP and the parallel composition of DP~\cite{McSherry10}. %
Note that we maintain a minimum batch size as required for $\SCovOutputPert$ so that we maintain DP.

Further, as in the analysis of \Cref{thm:convex-erm}, we can consider another algorithm $\cM'$ which is the same as $\cM$ except that on line \ref{line:output-pert-sco} it does not call $\SCovOutputPert$ but instead directly lets $\htheta_i \gets \theta^*_i(\bx) + \cN(0, \sigma_i^2 \cdot I)$ where $\sigma_i := \sigma(\eps',\delta',\beta', 2G, \lambda_i)$ is as in \Cref{thm:output-pert}. Proof of \Cref{thm:convex-sco} is completed by following the same analysis as in the proof of \Cref{thm:convex-erm} with \Cref{lem:phased-erm-intermediate} replaced with \Cref{lem:phased-sco-intermediate}.
\end{proof}

\subsection{Strongly Convex Losses}

We obtain better excess risk bounds for the case of strongly convex losses, as stated below. The proof is similar to that of DP-ERM for strongly convex loss, i.e., we use output perturbation and then run DP-SCO for convex losses (\Cref{thm:convex-sco}) using a smaller radius. An additional step here is to show that an empirical minimizer is $\tO\left(\frac{G}{\mu\sqrt{nm}}\right)$-close to the population minimizer w.h.p. (which might be of independent interest; see \Cref{prop:erm-closeness}). The full proof is presented in \Cref{subsec:strongly-convex-sco}.

\begin{theorem}\label{thm:strongly-convex-sco}
For any $G$-Lipschitz, $\mu$-strongly convex loss $\ell$, there exists an $(\eps, \delta)$-DP mechanism that, for all $n \geq \tO\inparen{\frac{\log(1/\delta)\log(m)}{\eps}}$, outputs $\htheta$ such that %
\begin{align*}
&\textstyle \E_{\pi, \htheta \gets \cM(\bx^{\pi})}[\cL(\htheta; \cD)] - \cL(\theta^*; \cD)\\
&\textstyle ~\leq~~ \tO_\eps\inparen{\frac{G^2}{\mu} \inparen{\frac{d}{n^2m} + \frac{1}{nm}}},
\end{align*}
where $\tO_{\eps}$ hides a multiplicative factor of $\poly(\log(1/\delta), \log(nm), 1/\eps)$.
\end{theorem}

\section{Discussion \& Open Problems}
\label{sec:conclusion}

Although we do not discuss the running time of our algorithm, it can be seen that they run in $n^{O(\log(1/\delta)/\eps)}(md)^{O(1)}$ time; the bottleneck comes from the step to compute $\cX^{R_1}_{\stable}$ in $\DelOutputPert$, which requires enumerating all subsets $S$ of size $R_1 = O\left(\frac{\log(1/\delta)}{\eps}\right)$. In \Cref{app:efficient}, we describe a speed up for all our DP-SCO/ERM results that makes the algorithm run in polynomial (in $n,m,d$) time \emph{with high probability}. However, with the remaining $o(1)$ probability, the algorithm may still take $n^{O(\log(1/\delta)/\eps)}(md)^{O(1)}$ time. It remains open whether we can get an algorithm whose running time is polynomial in the \emph{worst} case. As discussed in the introduction, it was not known whether excess risk bounds that we achieve were obtainable (even with inefficient algorithms) before. 

Another question is whether we can get tight dependency on $\delta, \eps$. Specifically, our ERM excess risk bound in the convex case has a factor of $O\left(\frac{(\log(1/\delta))^{2}}{\eps^{2.5}}\right)$,
while previous bounds only had  $\tO\left(\frac{\sqrt{\log(1/\delta)}}{\eps}\right)$. %
Note that our larger dependency is indeed due to the generic output perturbation algorithm ($\DelOutputPert$), which requires the noise scale $\sigma$ to be inflated by a factor of $\kappa = O\left(\frac{\log(1/\delta)}{\eps}\right)$, and the union bound performed for \Cref{cor:delsen-neighborhood-whp} which includes another $\sqrt{\kappa}$ factor. Therefore, this question may be related to the previous question.

\subsection*{Acknowledgements}
P.K. would like to thank Gene Li, Ohad Shamir and Nathan Srebro for helpful discussions about stochastic convex optimization. P.M. would also like to thank Adam Sealfon for useful discussions and for pointers to DP graph analysis literature.

\newpage

\balance
\bibliography{ref}
\bibliographystyle{icml2023}

\newpage
\appendix
\onecolumn

\section{Proof of \texorpdfstring{\Cref{lem:vector_concentration}}{Lemma~\ref{lem:vector_concentration}}}\label{apx:proof_vector_concentration}

\Cref{lem:vector_concentration} follows quite immediately as an application of a special case of Proposition 1 in \citet{SambaleS11} as stated below. Let $S_N$ denote the set of all permutations over $[N]$ and for any $\pi \in S_N$, let $\pi^{i \leftrightarrow j}$ denote the permutation with $i$th and $j$th entries swapped.

\begin{proposition}[Proposition 1 in \citet{SambaleS11}]\label{prop:perm_concentration}
Let $f : S_N \to \R$ be a real-valued function over $S_N$, such that $|f(\pi) - f(\pi^{i \leftrightarrow j})| \le c_{i,j}$ for all $\pi \in S_N$ and all $1 \le i < j \le N$ for some $c_{i,j} \ge 0$. For any $t \ge 0$, it holds that
\[
\Pr_{\pi \sim S_N}[f(\pi) - \E[ f(\pi)] \ge t] \le \exp\inparen{ -\frac{Nt^2}{4\sum_{1 \le i < j \le N} c_{i,j}^2}}.
\]
\end{proposition}

\begin{proof}[Proof of \Cref{lem:vector_concentration}]
Since $\sum_i \bv_i = 0$, we have for any two vectors $\bu, \bv$ sampled randomly without replacement from $\{\bv_1, \ldots,  \bv_N\}$ that $\langle \bu, \bv \rangle < 0$, since $\E[\bu \mid \bv] = - \bv / (N-1)$. Hence, we have
\begin{align*}
    \textstyle\E\left[\left\|\sum_{j \in [m]} \bv_{i_j} \right\|^2\right]
    &~=~ \sum_{j \in [m]}\E[\| \bv_{i_j}\|^2] + 2\,\sum_{j < k}\E\left[\left\langle \bv_{i_j}, \bv_{i_k}\right\rangle\right]
    ~\leq~ mG^2,
\end{align*}
and hence $\E\left[\left\|\sum_{j \in [m]} \bv_{i_j} \right\|\right] \le \sqrt{m}G$.
Let $f : S_N \to \R$ be defined as $f(\pi) = \|\sum_{j=1}^m \bv_{\pi(j)}\|$. It follows that $f(\pi) = f(\pi^{i \leftrightarrow j})$ whenever both $i, j \le m$ or both $i, j > m$. Moreover, when $i \le m$ and $j > m$, it holds that
\begin{align*}
|f(\pi) - f(\pi^{i \leftrightarrow j})|
&~=~ \left\| \sum_{k=1}^m \bv_{\pi(k)} \right\| - \left\| \sum_{k=1}^m \bv_{\pi^{i \leftrightarrow j}(k)} \right\|
~\le~ \left\| \bv_{\pi(i)} - \bv_{\pi(j)}\right\|
~\le~ 2G.
\end{align*}
Thus, using \Cref{prop:perm_concentration} with $c_{i,j} = 2G$ when $i \le m < j$ and $0$ otherwise, we have that
\begin{align*}
\textstyle \Pr_{\pi \sim S_N} \left[\left\|\sum_{j \in [m]} \bv_{\pi(j)} \right\| \ge t + \sqrt{m}G\right]
&~\le~ \exp\inparen{- \frac{Nt^2}{16 m(N-m)G^2}}
~\le~ \exp\inparen{- \frac{t^2}{16 mG^2}}.
\end{align*}
Choosing $t = 4G \sqrt{m \log(1/\beta)}$ completes the proof.
\end{proof}

\section{Proofs of Improved Bounds for Strongly Convex Losses}\label{app:strongly-convex}

\subsection{Empirical Risk Minimization}\label{subsec:strongly-convex-erm}

\begin{algorithm}[h]
\caption{$\StronglyConvexERM$}
\label{alg:dp-strongly-convex-erm}
\begin{algorithmic}[1]
\STATE \textbf{Input: } Dataset $\bx$, loss function $\ell: \Theta \times \cX \to \R$ that is $\mu$-strongly convex and $G$-Lipschitz
\STATE \textbf{Parameters: } Privacy parameters $\eps, \delta$;  Target Failure Probability $\beta$
\STATE $\theta_0 \gets \SCovOutputPert_{\eps/2, \delta/2, \beta, G, \mu, \cK}(\ell; \bx)$
\STATE $R' \gets \sigma(\eps/2, \delta/2, \beta, G, \mu) \cdot \sqrt{d \log 1/\beta}$
\STATE $\cK' \gets \cK \cap \cB_d(\theta_0, R')$
\STATE $\lambda \gets \frac{G \sqrt{d}}{R' n \sqrt{m}}$
\STATE $\htheta \gets \PhasedERM_{\eps/2, \delta/2, \beta, G, \lambda, \cK'}(\ell; \bx)$
\RETURN $\htheta$
\end{algorithmic}
\end{algorithm}

\begin{proof}[Proof of \Cref{thm:strongly-convex-erm}]
The mechanism in \Cref{alg:dp-strongly-convex-erm}, which uses a two-step approach to get stronger rates for strongly convex losses, following a similar reduction in \citet{BassilyST14}. It first uses the $\SCovOutputPert$ algorithm with $(\eps/2, \delta/2)$-DP, which with probability $1-\beta$ returns $\theta_0 := \theta^*(\bx) + e$ where $e \sim \cN(0, \sigma^2 \cdot I)$ for $\sigma$ specified in \Cref{thm:output-pert}. From standard concentration, we have that $\Pr[\|e\| \ge C \sigma \sqrt{d \log 1/\beta}] \le \beta$, for a suitable $C$. Thus, with probability $1 - 2\beta$, we have that $\theta^*(\bx)$ is indeed contained in $\cB_d(\theta_0, R')$ for $R' = C \sigma \sqrt{d \log 1/\beta} = \tO_{\eps}(G\sqrt{d}/(\mu n\sqrt{m}))$; note that this can be much smaller than the diameter of $\cK$ which is at most $2G/\mu$. Finally, we use the $\PhasedERM$ algorithm with $(\eps/2, \delta/2)$-DP over the region $\cK' = \cK \cap \cB_d(\theta_0, R')$. Following the proof of \Cref{thm:convex-erm}, setting $\beta = 1/2n^2 m$, we have that
\[
\textstyle
\E[\cL(\htheta; \bx)] - \cL(\theta^*; \bx) ~\le~ \tO_{\eps}\inparen{\frac{G^2 d}{\mu n^2 m}}.
\]
The value of $\beta$ was chosen such that $\beta RG \le O(G^2 / (\mu n^2 m))$, where $R$ is the diameter of $\cK$, which is at most $2G/\mu$. This is to account for the probability of at most $2\beta$ that either $\SCovOutputPert$ fails or that $\|e\| > C \sigma \sqrt{d \log 1/\beta}$, in which case, the excess risk is at most $RG$.
\end{proof}

\subsection{Stochastic Convex Optimization}\label{subsec:strongly-convex-sco}

We rely on the following proposition, which to the best of our knowledge, is not known in the literature.
\begin{proposition}\label{prop:erm-closeness}
For any $G$-Lipschitz, $\mu$-strongly convex loss $\ell$ and for any distribution $\cD$, it holds for all $\beta < 1/e$ that
\[
\Pr_{\bx \sim \cD^{n \times m}} \insquare{\|\theta^*(\bx) - \theta^*(\cD)\| \le \frac{30G\sqrt{\log(2/\beta)}}{\mu \sqrt{nm}}} \ge 1 - \beta.
\]
\end{proposition}

Before we prove \Cref{prop:erm-closeness}, let us see how to use it to prove \Cref{thm:strongly-convex-sco}.

\begin{algorithm}[h]
\caption{$\StronglyConvexSCO$}
\label{alg:dp-strongly-convex-sco}
\begin{algorithmic}[1]
\STATE \textbf{Input: } Dataset $\bx$, loss function $\ell: \Theta \times \cX \to \R$ that is $\mu$-strongly convex and $G$-Lipschitz
\STATE \textbf{Parameters: } Privacy parameters $\eps, \delta$; Target Failure Probability $\beta$
\STATE $\theta_0 \gets \SCovOutputPert_{\eps/2, \delta/2, \beta, G, \mu, \cK}(\ell; \bx)$
\STATE $R' \gets \sigma(\eps/2, \delta/2, \beta, G, \mu) \cdot \sqrt{d \log 1/\beta} + \frac{G\sqrt{\log 1/\beta}}{\mu \sqrt{nm}}$
\STATE $\cK' \gets \cK \cap \cB_d(\theta_0, R')$
\STATE $\lambda \gets \frac{G \sqrt{d}}{R' n \sqrt{m}}$
\STATE $\htheta \gets \PhasedSCO_{\eps/2, \delta/2, \beta, G, \lambda, \cK'}(\ell; \bx)$
\RETURN $\htheta$
\end{algorithmic}
\end{algorithm}

\begin{proof}[Proof of \Cref{thm:strongly-convex-sco}]
\Cref{alg:dp-strongly-convex-sco}  is similar to \Cref{alg:dp-strongly-convex-erm}, namely, it first uses the $\SCovOutputPert$ algorithm with $(\eps/2, \delta/2)$-DP, which with probability $1-\beta$ returns $\theta_0 := \theta^*(\bx) + e$ where $e \sim \cN(0, \sigma^2 \cdot I)$. Using \Cref{prop:erm-closeness}, we have that with probability at least $1-\beta$, it holds that $\|\theta^*(\bx) - \theta^*(\cD)\| \le O(G\sqrt{\log 1/\beta} / (\mu \sqrt{nm}))$. Thus, we have that $\theta^*(\cD)$ is contained in $\cB_d(\theta_0, R')$ for $R' = O\inparen{\frac{G}{\mu} \inparen{\frac{\sqrt{d}}{n \sqrt{m}} + \frac{1}{\sqrt{nm}}}}$ with probability at least $1 - \beta$. Finally, we use the $\PhasedSCO$ algorithm with $(\eps/2, \delta/2)$-DP over the region $\cK' = \cK \cap \cB_d(\theta_0, R')$. We get our desired conclusion by plugging in the bound for $R'$ in \Cref{thm:convex-sco}, again setting $\beta = 1/2n^2 m$.
\end{proof}

We suspect that \Cref{prop:erm-closeness} might already be known in literature. Since we are unaware of a reference, we include a proof for completeness, which incidentally uses our new result about deletion stability (\Cref{thm:deletion-stability}).

\begin{proof}[Proof of \Cref{prop:erm-closeness}]
First, it is well known from \citet{Shalev-ShwartzSSS09} that
\[
\E_{\bx \sim \cD^{n \times m}} [\cL(\theta^*(\bx); \cD)] - \cL(\theta^*(\cD); \cD) ~\le~ \frac{4G^2}{\mu nm}.
\]
On the other hand, from strong convexity we have for all $\bx$ that
\[
\cL(\theta^*(\bx); \cD) - \cL(\theta^*(\cD); \cD) ~\ge~ \frac{\mu}{2} \cdot \|\theta^*(\bx) - \theta^*(\cD)\|^2.
\]
Combining the above, we have
\begin{align}
\E_{\bx \sim \cD^{n \times m}} [\|\theta^*(\bx) - \theta^*(\cD)\|]
&~\le~ \frac{3G}{\mu \sqrt{nm}}. \label{eq:expected-erm-close-to-optimal}
\end{align}
Additionally, from \Cref{thm:deletion-stability} (invoked twice with $m \gets nm$ and $n \gets 2$, followed by the triangle inequality and a union bound), it follows that
\begin{align*}
    \Pr_{\bx, \bx' \sim \cD^{n \times m}} \insquare{\|\theta^*(\bx) - \theta^*(\bx')\| ~\leq~ \frac{10G\sqrt{\log(2/\beta)}}{\mu \sqrt{nm}}} &~\le~ \beta.\nonumber
\end{align*}
By an averaging argument, there exists $\theta_0 = \theta^*(\bx^{(0)})$ for some $\bx^{(0)}$, such that
\begin{align*}
    \Pr_{\bx \sim \cD^{n \times m}} \insquare{\|\theta^*(\bx) - \theta_0\| ~\leq~ \frac{10G\sqrt{\log(2/\beta)}}{\mu \sqrt{nm}}} &~\le~ \beta.
\end{align*}
Thus, combining with \Cref{eq:expected-erm-close-to-optimal}, we have
\begin{align*}
\E_{\bx \sim \cD^{n \times m}} \|\theta^*(\bx) - \theta^*(\cD)\|
&~\ge~ (1-\beta) \cdot \inparen{\|\theta_0 - \theta^*(\cD)\| - \frac{10G\sqrt{\log(2/\beta)}}{\mu \sqrt{nm}}}\\
\Longrightarrow\qquad
\|\theta_0 - \theta^*(\cD)\| &~\le~ \frac{20G\sqrt{\log(2/\beta)}}{\mu \sqrt{nm}} \qquad \text{(for $\beta < 1/2$)}.
\end{align*}
Finally by the triangle inequality, we get
\[
\Pr_{\bx \sim \cD^{n \times m}} \insquare{\left\|\theta^*(\bx) - \theta^*(\cD)\right\| ~\leq~ \frac{30G\sqrt{\log(2/\beta)}}{\mu \sqrt{nm}}} ~\le~ \beta.\qedhere
\]
\end{proof}

\section{On Speeding up our Algorithms}
\label{app:efficient}

As stated in \Cref{sec:conclusion}, the time bottleneck of our algorithm is $\DelOutputPert$, which requires computing 
 $\cX^{R_1}_{\stable}$. Doing this in a straightforward manner requires enumerating all sets $S$ of size $R_1$, resulting in a running time of $n^{R_1}(md)^{O(1)} = n^{O(\log(1/\delta)/\eps)}(md)^{O(1)}$. In this section, we sketch an argument that brings the  time down to $(nmd)^{O(1)}$ with high probability, while maintaining all the error bounds to within $\tO_\eps(1)$ factor. Note that all algorithms invoke $\DelOutputPert$ through \Cref{thm:output-pert} (i.e., the $\SCovOutputPert$ algorithm). Therefore, it suffices to argue how to achieve the speed up for $\SCovOutputPert$.

The first observation here is that if $\bx$ belongs to $\cX^{R_1}_{\stable}$, then we can just output $\theta^*(\bx) + \cN(0, \sigma^2 \cdot I)$. Furthermore, we have already shown (\Cref{thm:output-pert}) that $\bx \in \cX^{R_1}_{\stable}$ with high probability. Thus, if we can give a ``certificate'' that $\bx \in \cX^{R_1}_{\stable}$, then we would be able to complete skip the check and just output $\theta^*(\bx) + \cN(0, \sigma^2 \cdot I)$; this means that, whenever we have such a certificate, our algorithm will run in polynomial (in $n,m,d$) time.

Our certificate is simple: the gradients at $\theta^*$ w.r.t. each user. The following lemma (whose proof is similar to part of the proof of \Cref{thm:deletion-stability}) relates this certificate to $\delsen_r$ (which in turn implies membership in $\cX^{R_1}_{\stable}$ for appropriate $\Delta$).

\begin{lemma}
Let $\bx$ be any dataset and let $\theta^* := \theta^*(\bx)$. Suppose that for all $i \in [n]$, we have $\left\|\nabla\cL(\theta^*; \bx_i)\right\| \leq \gamma$. Then, we have $\delsen_r \theta^*(\bx) \leq \Delta$ for $\Delta = O(\frac{r \gamma}{\mu n})$ for all $r \leq n/2$.
\end{lemma}

\begin{proof}
Consider any set $S \subseteq [n]$ such that $|S| \leq r$. Let $s := |S|$ and $\theta^*_{-S} := \theta^*(\bx_{-S})$. Since $\nabla \cL(\theta^*; \bx) = 0$, we have
\begin{align*}
\left\|\nabla \cL(\theta^*; \bx_{-S})\right\|
&= \left\|\frac{1}{n-s} \sum_{i \in S} \nabla \cL(\theta^*; \bx_i)\right\| \\
&\leq \frac{1}{n-s} \sum_{i \in S} \left\|\nabla \cL(\theta^*; \bx_i)\right\| \\
&\leq \frac{s \gamma}{n-s} \\
&\leq \frac{r \gamma}{n/2} = O(r\gamma / n).
\end{align*}

Therefore, we have
\begin{align*}
\cL(\theta^*; \bx_{-S}) - \cL(\theta^{*}_{-S}; \bx_{-S}) &~\leq~ \left<\nabla \cL(\theta^*; \bx_{-S}), \theta^* - \theta^*_{-S}\right> \\
&\leq O(r\gamma / n) \cdot \|\theta^* - \theta^*_{-S}\|.
\end{align*}
On the other hand, since $\ell$ is $\mu$-strongly convex and since $\theta^*_{-S}$ is the minimizer for $\cL(\cdot; \bx_{-S})$, we can conclude that
\begin{align*}
\cL(\theta^*; \bx_{-S}) - \cL(\theta^{*}_{-S}; \bx_{-S}) ~\geq~ \frac{\mu}{2} \|\theta^{*}_{-S} - \theta^*\|^2
\end{align*}
Comparing the two bounds above, we get 
\begin{align*}
\|\theta^* - \theta^{*,\pi}_{-n}\| ~\leq~ O\left(\frac{r \gamma}{\mu n}\right). & \qedhere
\end{align*}
\end{proof}

Recall also from the proof of \Cref{thm:deletion-stability} that w.h.p. we have $\left\|\nabla\cL(\theta^*; \bx_i)\right\| \leq \tO(G/\sqrt{m})$. When this holds, by computing $\sum_{j \in [m]} \nabla\ell(\theta^*; \bx_{i, j})$ for all $i \in [n]$, the above lemma means that this is a certificate that $\bx \in \cX^{R_1}_{\stable}$ when we set $\Delta = O(\frac{\kappa \gamma}{\mu n}) = \tO\left(\frac{G \cdot \log(1/\delta)}{\eps \mu n\sqrt{m}}\right)$. Plugging this into~\Cref{thm:output-pert-generic}, we arrive at a statement similar to \Cref{thm:output-pert} but with $$\sigma = O\inparen{\frac{G\sqrt{\log n \log(1/\delta) / \eps + \log(1/\beta)}}{\mu n \sqrt{m}} \cdot \frac{(\log(1/\delta))^{2.5}}{\eps^3}},$$
i.e., with an extra factor of $O(\log(1/\delta)/\eps)$. On the other hand, from the discussion about the certificate, we have that this algorithm runs in polynomial time with high probability (whenever $\left\|\nabla\cL(\theta^*; \bx_i)\right\| \leq \tO(\sqrt{m})$).

\newcommand{\lin}{\mathrm{lin}}
\newcommand{\sq}{\mathrm{sq}}
\newcommand{\tr}{\mathrm{tr}}
\newcommand{\userdp}{\mathrm{user}}
\newcommand{\itemdp}{\mathrm{item}}
\newcommand{\A}{\mathcal{A}}
\newcommand{\tD}{\tilde{D}}
\newcommand{\tcD}{\tilde{\cD}}
\newcommand{\hchi}{\hat{\chi}}

\section{On Lower Bounds for User-Level DP-ERM and DP-SCO}
\label{app:lb}

This section discusses lower bounds for user-level DP-SCO and DP-ERM. We start by noting that \citet{LevySAKKMS21} already proved a lower bound of $\Omega\left(RG\left(\frac{1}{\sqrt{nm}} + \frac{\sqrt{d}}{\eps n \sqrt{m}}\right)\right)$ for DP-SCO for the convex case assuming $n \geq \Omega\left(\sqrt{d}/\eps\right)$. It can be easily seen that this also implies a lower bound of $\Omega\left(RG \cdot \frac{\sqrt{d}}{\eps n \sqrt{m}}\right)$ for $\Omega\left(\sqrt{d}/\eps\right) \leq n \leq O\left(d/\eps^2\right)$ (see, e.g., the proof of \Cref{thm:lb-sc-erm} below). In the remainder of this section, we extend their techniques to show the lower bounds for strongly convex losses.

\subsection{Preliminaries}

Throughout, we will consider the loss $\ell_{\sq}^{\zeta}(\theta; x) := \zeta \cdot \|\theta - x\|^2$ where $\zeta > 0$ is a parameter. We list here a few results that will be useful throughout. We start by defining the ($\ell_2$-)truncated version of the Gaussian distribution as follows.

\begin{definition}
Let $\cN^{\tr}(\chi, \Sigma; B)$ denote the distribution of r.v. $Z$ drawn as follows. First, draw $Z' \sim \cN(\chi, \Sigma)$. Then, let $Z = Z' \cdot \ind[\|Z'\| \leq B]$. We use $\chi^{\tr}(\chi, \Sigma; B)$ to denote the mean of the distribution $\cN^{\tr}(\chi, \Sigma; B)$.
\end{definition}

As shown in \citet{LevySAKKMS21}, the means of the truncated Gaussian distribution and the standard (non-truncated) version are very close:
\begin{lemma}[{\citealt{LevySAKKMS21}}] \label{lem:trunc-mean-change}
For any $\chi \in \R^d, d \in \N, \sigma > 0$, if $\|\chi\|_2 + 100 \sqrt{d} \cdot \sigma < B$, then $\|\chi^{\tr}(\chi, \sigma^2 I_d; B) - \chi\|_2 \leq O((\sigma + \|\chi\|_2) \cdot e^{-10d})$.
\end{lemma}

Since the version of the lemma in \citealt{LevySAKKMS21} is slightly different than the one we use here, we give a proof sketch of this below\footnote{More precisely, \citealt{LevySAKKMS21} is using truncation in a coordinate-by-coordinate manner (i.e. by the $\ell_\infty$ norm), which results in an extra polylogarithmic factor.}.

\begin{proof}[Proof Sketch of \Cref{lem:trunc-mean-change}]
Due to spherical symmetry, we may assume w.l.o.g. that $\chi_2 = \cdots = \chi_d = 0$ and $\chi_1 \geq 0$. Again, due to symmetry, we have $\chi^{\tr}(\chi, \sigma^2 I_d; B)_2 = \cdots = \chi^{\tr}(\chi, \sigma^2 I_d; B)_d = 0$ and thus $\|\chi^{\tr}(\chi, \sigma^2 I_d; B) - \chi\|_2 = |\chi^{\tr}(\chi, \sigma^2 I_d; B)_1 - \chi_1|$.

To bound this term, observe further that we may view $Z_1$ as being generated as follows:
\begin{itemize}
\item Sample $Z'_1 \sim \cN(\chi_1, \sigma^2)$.
\item Sample $U \sim \chi^2(d - 1)$ . (This represents $((Z'_2)^2 + \cdots + (Z'_d)^2)/\sigma^2$.)
\item Let $Z_1 = Z'_1 \cdot \ind[(Z'_1)^2 + \sigma^2 \cdot U \leq B^2]$
\end{itemize}
For $u > 0$, let $\mu_u$ denote the mean of $Z_1$ conditioned on $U = u$. We have
\begin{align*}
\chi^{\tr}(\chi, \sigma^2 I_d; B) = \E_{U \sim \chi^2(d - 1)}[\mu_U]. 
\end{align*}
From symmetry, it is again simple to see that $0 \leq \mu_U \leq \chi_1$.  As such, we have
\begin{align*}
|\chi^{\tr}(\chi, \sigma^2 I_d; B)_1 - \chi_1| \leq \E_{U \sim \chi^2(d - 1)}[|\mu_U - \chi_1|] = \E_{U \sim \chi^2(d - 1)}[\chi_1 - \mu_U].
\end{align*}
Now, using standard concentration of $\chi^2(d - 1)$ distribution (see e.g.,~\cite{chi-sq}), we have $\Pr_{U \sim \chi^2(d - 1)}[U \geq 70\sqrt{d}] \leq e^{-10d}$. From this, we have
\begin{align*}
|\chi^{\tr}(\chi, \sigma^2 I_d; B)_1 - \chi_1| &\leq \E_{U \sim \chi^2(d - 1)}[\chi_1 - \mu_U \mid U \leq 70\sqrt{d}] + \chi_1 \cdot \Pr_{U \sim \chi^2(d - 1)}[U \geq 70\sqrt{d}]
\\
&\leq \max_{u \in [0, 70\sqrt{d}]} \left(\chi_1 - \mu_u\right) + \chi_1 \cdot e^{-10d}.
\end{align*}
To bound the first term, observe that for a fixed $u$, we simply have $Z_1 = Z'_1 \ind[|Z'_1| \leq B_u]$ where $B_u := \sqrt{B^2 - \sigma^2 u} \geq \|\chi\|_2 + 70\sqrt{d} \cdot \sigma$. Thus, we have
\begin{align*}
\mu_u = \Pr[|Z'_1| \leq B_u] \E[Z'_1 \mid |Z'_1| \leq B_u] \geq (1 - e^{-10d}) \cdot \E[Z'_1 \mid |Z'_1| \leq B_u],
\end{align*}
where the probability bound on $\Pr[|Z'_1| > B_u]$ is based on standard concentrations of a (single-variate) Gaussian.

Finally, $\E[Z'_1 \mid |Z'_1| \leq B_u]$ is simply the expectation of the truncated single-variate Gaussian distribution, which has a closed-form formula, described below. Here $\psi, \Phi$ denote the PDF and CDF of the standard normal distribution respectively, and let $\alpha = \left(\frac{-B_u - \chi_1}{\sigma}\right), \beta = \left(\frac{B_u - \chi_1}{\sigma}\right)$. Note that we have $\beta \geq 70\sqrt{d}$.
\begin{align*}
\E[Z'_1 \mid |Z'_1| \leq B_u] = \chi_1 + \sigma\left(\frac{\psi(\alpha) - \psi(\beta)}{\Phi(\beta) - \Phi(\alpha)}\right)
\geq \chi_1 - \sigma \cdot O(\psi(\beta)) \geq \chi_1 - \sigma \cdot O(e^{-10 d}).
\end{align*}
Plugging the previous three bounds together, we have
\begin{align*}
|\chi^{\tr}(\chi, \sigma^2 I_d; B)_1 - \chi_1| &\leq O((\sigma + \chi_1) \cdot e^{-10d}).
\qedhere
\end{align*}
\end{proof}

More importantly, \citet{LevySAKKMS21} make the following crucial observation, which allows us to reduce any user-level DP algorithm for Gaussian distribution back to an item-level DP algorithm, albeit with the variance that is $m$ times smaller.

\begin{lemma}[User-to-Item Level Reduction, {\citealt{LevySAKKMS21}}] \label{lem:user-to-item-red}
Let $\A_{\userdp}$ be any user-level $(\eps,\delta)$-DP algorithm. Then, there exists an item-level  $(\eps, \delta)$-DP algorithm $\A_{\itemdp}$ such that, for any Gaussian distribution $\cD = \cN(\chi, \sigma^2 I_d)$, $\A_{\userdp}(\cD^{n \times m})$ has exactly the same distribution as $\A_{\itemdp}(\tcD^n)$ where $\tcD = \cN\left(\chi, \frac{\sigma^2}{m} I_d\right)$.
\end{lemma}

Finally, we will use the following ``fingerprinting lemma for Gaussians'' result due to \citet{KamathSU19}, which gives a lower bound for any DP algorithm for estimating the mean of a spherical Gaussian.

\begin{theorem}[\citealt{KamathSU19}] \label{thm:item-level-gaussian-mean-est}
For any $\psi  \in (0, 1), \sigma > 0, n, d \in \N$ and $\eps \in (0, 1], \delta \in (0, 1/2]$ such that $\delta \leq \frac{\sqrt{d}}{100\psi n\sqrt{\log(100\psi n/\sqrt{d})}}$, if there exists an item-level $(\eps, \delta)$-DP mechanism $\cM$ such that, for any Gaussian distribution $\cD = \cN(\chi, \sigma^2 I_d)$ where $\chi$ is unknown with $-\psi \sigma \leq \chi \leq \psi \sigma$ it satisfies
\begin{align*}
\E_{\hchi \gets \cM(\cD^n)}[\|\hchi - \chi\|^2] \leq \alpha^2 \leq \frac{d\sigma^2 \psi ^2}{6},
\end{align*}
then we must have $n \geq \frac{d\sigma}{24\alpha\eps}$.
\end{theorem}

Combining \Cref{lem:user-to-item-red} and \Cref{thm:item-level-gaussian-mean-est}, we immediately arrive at the following lower bound for the user-level DP setting.

\begin{lemma}
For any $\psi  \in (0, 1), \sigma > 0, m, n, d \in \N$ and $\eps \in (0, 1], \delta \in (0, 1/2]$ such that $\delta \leq \frac{\sqrt{d}}{100\psi n\sqrt{\log(100\psi n/\sqrt{d})}}$, if there exists a user-level $(\eps, \delta)$-DP mechanism $\cM$ such that, for any Gaussian distribution $\cD = \cN(\chi, \sigma^2 I_d)$ where $\chi$ is unknown with $-\frac{\psi \sigma}{\sqrt{m}} \leq \chi \leq \frac{\psi \sigma}{\sqrt{m}}$ it satisfies
\begin{align*}
\E_{\hchi \gets \cM(\cD^{n \times m})}[\|\hchi - \chi\|^2] \leq \alpha^2 \leq \frac{d\sigma^2 \psi ^2}{6m},
\end{align*}
then we must have $n \geq \frac{d\sigma}{24\alpha\eps\sqrt{m}}$.
\end{lemma}

Furthermore, combining the above with \Cref{lem:trunc-mean-change}, we arrive at the following lower bound where the only change is from Gaussian distributions to truncated Gaussian distributions.

\begin{lemma} \label{lem:mean-est-lb-trunc}
For any $\psi  \in (\Omega(e^{-d}), 1), B, \sigma > 0, m, n, d \in \N$ and $\eps \in (0, 1], \delta \in (0, 1/2]$ such that $\delta \leq \frac{\sqrt{d}}{100\psi n\sqrt{\log(100\psi n/\sqrt{d})}}$ and $B > \frac{\psi \sigma\sqrt{d}}{\sqrt{m}} + 100 \sqrt{d} \sigma$, if there exists a user-level $(\eps, \delta)$-DP mechanism $\cM$ such that, for any truncated Gaussian distribution $\cD = \cN^{\tr}(\chi, \sigma^2 I_d; B)$ where $\chi$ is unknown with $-\frac{\psi \sigma}{\sqrt{m}} \leq \chi \leq \frac{\psi \sigma}{\sqrt{m}}$ it satisfies
\begin{align*}
\E_{\hchi \gets \cM(\cD^{n \times m})}[\|\hchi - \chi^{\tr}(\chi, \sigma^2I_d; B)\|^2] \leq \alpha^2 \leq \frac{d\sigma^2 \psi ^2}{12m},
\end{align*}
then we must have $n \geq \frac{d\sigma}{50\alpha\eps\sqrt{m}}$.
\end{lemma}

\subsection{Lower Bounds for Strongly Convex Losses}

\subsubsection{DP-SCO}

We can now prove the lower bound for DP-SCO in the strongly convex case in a relatively straightforward manner, as optimizing for the loss $\ell_{\sq}$ is equivalent to mean estimation with $\ell_2^2$-error.

\begin{theorem} \label{thm:lb-sc-sco}
For any $\eps \in (0, 1], \delta \in (0, 1/2]$ and any sufficiently large $d, n, m \in \N$ such that $n \geq \sqrt{d}/\eps$ and $\delta \leq \frac{\sqrt{d}}{200\sqrt{n\sqrt{\log n}}}$, there exists a $\mu$-strongly convex $G$-Lipschitz loss function $\ell$ such that for any $(\eps, \delta)$-DP algorithm, we have
\begin{align*}
\sup_\cD \left(\E_{\htheta \gets \cM(\cD^n)}\left[\cL(\htheta; \cD)\right] - \cL(\theta^*; \cD)\right) ~\geq~ \Omega\left(\frac{G^2}{\mu}\left(\frac{1}{nm} + \frac{d}{\eps^2 n^2 m}\right)\right).
\end{align*}
\end{theorem}

We note that the condition $n \geq \sqrt{d}/\eps$ may be unnecessary. However, a slightly weaker condition $n\sqrt{m} \geq \Omega(\sqrt{d}/\eps)$ is necessary because outputting, e.g., the origin already achieves an error of $G^2 / \mu$. Therefore, the second term $\frac{G^2}{\mu} \cdot \frac{d}{\eps^2 n^2 m}$ cannot be present in this case.

\begin{proof}[Proof of \Cref{thm:lb-sc-sco}]
The first term of $\Omega\left(\frac{G^2}{\mu} \frac{1}{nm}\right)$ is simply the statistical excess risk bound that holds even without any privacy considerations~\citep{AgarwalBRW12}.
We will only focus on the second term here.

Consider $\ell = \ell_{\sq}^{\zeta}$ for $\zeta = \mu/2$ and the parameter space $\cK = \cB_d(0, G/\mu)$. The loss is $\mu$-strongly convex and is $G$-Lipschitz in $\cK$. Set the parameters as follows: $B = \frac{G}{\mu}$, $\sigma = \frac{B}{1000\sqrt{d}}, \psi  = 1$. Let $\alpha = \frac{d\sigma}{100 \eps n \sqrt{m}}$; note that when $n \geq \sqrt{d}/\eps$, we also have $\alpha^2 \leq \frac{d\sigma^2\psi ^2}{12m}$ as desired.
Thus, we may apply \Cref{lem:mean-est-lb-trunc} with these parameters. This implies that, for any user-level $(\eps, \delta)$-DP mechanism $\cM$, there must be some truncated Gaussian distribution $\cD = \cN(\chi, \sigma^2I_d; B)$ such that
\begin{align*}
\E_{\hchi \gets \cM(\cD^{n \times m})}[\|\hchi - \chi^{\tr}(\chi, \sigma^2I_d; B)\|^2] \geq \Omega(\alpha^2) = \Omega\left(\frac{G^2}{\mu^2} \cdot \frac{d}{\eps^2n^2m}\right).
\end{align*}
Moreover, the excess (population) risk can be expanded as
\begin{align*}
\E_{\htheta \gets \cM(\cD^{n \times m})}\left[\cL(\htheta; \cD)\right] - \cL(\theta^*; \cD) = \frac{\mu}{2} \cdot \E_{\htheta \gets \cM(\cD^{n \times m})}[\|\htheta - \chi^{\tr}(\chi, \sigma^2I_d; B)\|^2] &\geq \Omega\left(\frac{G^2}{\mu} \cdot \frac{d}{\eps^2 n^2 m}\right). & &\qedhere
\end{align*}
\end{proof}

\subsubsection{DP-ERM}

The proof for DP-ERM is similar to above, except that we now have to account for the error between the population mean and the empirical mean. We enforce the parameters in such a way that this error is dominated by the lower bound given by \Cref{thm:lb-sc-sco}.

\begin{theorem} \label{thm:lb-sc-erm}
There exists a sufficiently small constant $c > 0$ such that the following holds.
For any $\eps \in (0, 1], \delta \in (0, 1/2]$ and any sufficiently large $d, n, m \in \N$ such that $c d / \eps^2 \leq n \geq \sqrt{d}/\eps$ and $\delta \leq \frac{\sqrt{d}}{200\sqrt{n\sqrt{\log n}}}$, there exists a $\mu$-strongly convex $G$-Lipschitz loss function $\ell$ such that for any $(\eps, \delta)$-DP algorithm, we have
\begin{align*}
\sup_\cD \left(\E_{\bx \gets \cD^{n \times m}, \htheta \gets \cM(\bx)}\left[\cL(\htheta; \bx) - \cL(\theta^*; \bx)\right]\right) &\geq \Omega\left(\frac{G^2}{\mu} \cdot \frac{d}{\eps^2 n^2 m}\right).
\end{align*}
\end{theorem}

In addition to the assumption $n \geq \sqrt{d}/\eps$ as in \Cref{thm:lb-sc-sco}, this theorem also requires the assumption $n \leq O(d / \eps^2)$. This assumption is required for the error between the empirical mean and the population mean to be small enough to be dominated by the error term $\Omega\left(\frac{G^2}{\mu} \cdot \frac{d}{\eps^2 n^2 m}\right)$.

\begin{proof}[Proof of \Cref{thm:lb-sc-erm}]
Let $\ell, B, \sigma, \psi , \alpha$ be exactly as in the setting of \Cref{thm:lb-sc-sco}. Similarly, there must exist some truncated Gaussian distribution $\cD = \cN(\chi, \sigma^2I_d; B)$ such that, for any $(\eps, \delta)$-DP algorithm $\cM$, we have
\begin{align*}
\E_{\hchi \gets \cM(\cD^{n \times m})}[\|\hchi - \chi^{\tr}(\chi, \sigma^2I_d; B)\|^2] \geq \Omega(\alpha^2) = \Omega\left(\frac{G^2}{\mu^2} \cdot \frac{d}{\eps^2n^2m}\right).
\end{align*}
Let $\hchi(\bx)$ denote the empirical mean of the dataset $\bx$. The left hand side can be further expanded as
\begin{align*}
\lefteqn{
\E_{\hchi \gets \cM(\bx), \bx \gets \cD^{n \times m}}[\|\hchi - \chi^{\tr}(\chi, \sigma^2I_d; B)\|^2]} \\
&\leq 2\left(\E_{\hchi \gets \cM(\bx), \bx \gets \cD^{n \times m}}[\|\hchi - \hchi(\bx)\|^2] + \E_{\bx \gets \cD^{n \times m}}[\|\hchi(\bx) - \chi^{\tr}(\chi, \sigma^2I_d; B)\|^2]\right) \\
&= 2 \E_{\hchi \gets \cM(\bx), \bx \gets \cD^{n \times m}}[\|\hchi - \hchi(\bx)\|^2] + O\left(\frac{B^2}{nm}\right) \\
&= 2 \E_{\hchi \gets \cM(\bx), \bx \gets \cD^{n \times m}}[\|\hchi - \hchi(\bx)\|^2] + O\left(\frac{G^2}{\mu^2} \cdot \frac{1}{nm}\right).
\end{align*}
Since we assume that $n \leq  c d / \eps^2$, we have $\frac{1}{nm} \leq c \cdot \frac{d}{\eps^2 n^2 m}$. Therefore, when $c$ is sufficiently small, we can combine the previous two inequalities to conclude that
\begin{align} \label{eq:emp-mean-acc}
\E_{\hchi \gets \cM(\bx), \bx \gets \cD^{n \times m}}[\|\chi - \hchi(\bx)\|^2] \geq \Omega\left(\frac{G^2}{\mu^2} \cdot \frac{d}{\eps^2n^2m}\right).
\end{align}

Finally, the excess (empirical) risk can be expanded as
\begin{align*}
\E_{\htheta \gets \cM(\bx), \bx \gets \cD^{n \times m}}\left[\cL(\htheta; \bx) - \cL(\theta^*; \bx)\right] = \frac{\mu}{2} \cdot \E_{\htheta \gets \cM(\bx), \bx \gets \cD^{n \times m}}[\|\htheta - \hchi(\bx)\|^2] \overset{\text{\eqref{eq:emp-mean-acc}}}{\geq} \Omega\left(\frac{G^2}{\mu} \cdot \frac{d}{\eps^2 n^2 m}\right). & & & \qedhere
\end{align*}
\end{proof}

\end{document}